\documentclass{article}

\usepackage[accepted]{configuration/icml2021}
\definecolor{darkblue}{rgb}{0.0, 0.0, 0.55}
\usepackage[hyperindex,breaklinks,colorlinks,citecolor=darkblue]{hyperref} 

\usepackage[utf8]{inputenc}
\usepackage[T1]{fontenc}
\usepackage{booktabs}
\usepackage{amsfonts}       
\usepackage{nicefrac}       
\usepackage{microtype}      
\usepackage{xcolor}
\usepackage[flushleft]{threeparttable}

\usepackage{url}
\usepackage{float}
\usepackage{subfigure}
\usepackage{graphicx}
\usepackage{multirow}
\usepackage{xspace}
\usepackage{natbib}
\usepackage{enumitem}
\usepackage[font=small]{caption}
\usepackage{diagbox}
\usepackage{wrapfig}
\usepackage[toc, page, header]{appendix}
\usepackage{amsmath,amsthm,amssymb}
\newtheorem{theorem}{Theorem}[section]
\newtheorem{lemma}[theorem]{Lemma}

\newtheorem{proposition}[theorem]{Proposition}

\newtheorem{assumption}{Assumption}

\providecommand{\norm}[1]{\left\lVert#1\right\rVert}

\providecommand{\R}{\mathbb{R}} 

\providecommand{\E}{{\mathbb E}}
\providecommand{\E}[1]{{\mathbb E}\left.#1\right. }        
\providecommand{\EE}[2]{{\mathbb E}_{#1}\left.#2\right. }  
\providecommand{\EEb}[2]{{\mathbb E}_{#1}\left[#2\right] } 

\providecommand{\1}{\mathbf{1}}

\providecommand{\xx}{\mathbf{x}}
\providecommand{\yy}{\mathbf{y}}

\providecommand{\mA}{\mathbf{A}}
\providecommand{\mB}{\mathbf{B}}

\providecommand{\mI}{\mathbf{I}}

\providecommand{\mU}{\mathbf{U}}

\providecommand{\mW}{\mathbf{W}}
\providecommand{\mX}{\mathbf{X}}
\providecommand{\mY}{\mathbf{Y}}

\providecommand{\mLambda}{\boldsymbol{\Lambda}}

\providecommand{\cD}{\mathcal{D}}

\providecommand{\cO}{\mathcal{O}}

\usepackage{algorithm}
\usepackage[noend]{algpseudocode}

\errorcontextlines\maxdimen

\makeatletter
\newcommand*{\algrule}[1][\algorithmicindent]{\makebox[#1][l]{\hspace*{.5em}\thealgruleextra\vrule height \thealgruleheight depth \thealgruledepth}}%
\newcommand*{\thealgruleextra}{}
\newcommand*{\thealgruleheight}{.75\baselineskip}
\newcommand*{\thealgruledepth}{.25\baselineskip}

\newcount\ALG@printindent@tempcnta
\def\ALG@printindent{%
	\ifnum \theALG@nested>0
	\ifx\ALG@text\ALG@x@notext
	\else
		\unskip
		\addvspace{-1pt}
		\ALG@printindent@tempcnta=1
		\loop
		\algrule[\csname ALG@ind@\the\ALG@printindent@tempcnta\endcsname]%
		\advance \ALG@printindent@tempcnta 1
		\ifnum \ALG@printindent@tempcnta<\numexpr\theALG@nested+1\relax
			\repeat
		\fi
	\fi
}%
\usepackage{etoolbox}
\patchcmd{\ALG@doentity}{\noindent\hskip\ALG@tlm}{\ALG@printindent}{}{\errmessage{failed to patch}}
\makeatother

\newbox\statebox
\newcommand{\myState}[1]{%
	\setbox\statebox=\vbox{#1}%
	\edef\thealgruleheight{\dimexpr \the\ht\statebox+1pt\relax}%
	\edef\thealgruledepth{\dimexpr \the\dp\statebox+1pt\relax}%
	\ifdim\thealgruleheight<.75\baselineskip
		\def\thealgruleheight{\dimexpr .75\baselineskip+1pt\relax}%
	\fi
	\ifdim\thealgruledepth<.25\baselineskip
		\def\thealgruledepth{\dimexpr .25\baselineskip+1pt\relax}%
	\fi
	\State #1%
	\def\thealgruleheight{\dimexpr .75\baselineskip+1pt\relax}%
	\def\thealgruledepth{\dimexpr .25\baselineskip+1pt\relax}%
}

\usepackage[textwidth=5cm]{todonotes}

\usepackage{xspace}

\newcommand{\Ximax}{\Xi_{\text{max}}} 
\newcommand{\phiavg}{\phi^{\text{avg}}} 
\newcommand{\phiema}{\phi^{\text{ema}}}

\icmltitlerunning{Consensus Control for Decentralized Deep Learning}

\begin{document}

\twocolumn[
	\icmltitle{Consensus Control for Decentralized Deep Learning}

	\icmlsetsymbol{equal}{*}

	\begin{icmlauthorlist}
		\icmlauthor{Lingjing Kong}{to,equal}
		\icmlauthor{Tao Lin}{to,equal}
		\icmlauthor{Anastasia Koloskova}{to}
		\icmlauthor{Martin Jaggi}{to}
		\icmlauthor{Sebastian U. Stich}{to}
	\end{icmlauthorlist}

	\icmlaffiliation{to}{EPFL, Lausanne, Switzerland}

	\icmlcorrespondingauthor{Tao Lin}{tao.lin@epfl.ch}

	\icmlkeywords{Machine Learning, ICML}

	\vskip 0.3in
]


\printAffiliationsAndNotice{\icmlEqualContribution} 

\setlength{\parskip}{1.5mm}  


\begin{abstract}
	Decentralized training of deep learning models enables on-device learning over networks,
	as well as efficient scaling to large compute clusters.
	Experiments in earlier works reveal that, even in a data-center setup, decentralized training
	often suffers from the degradation in the quality of the model:\
	the training and test performance of models trained in a decentralized fashion is in general worse
	than that of models trained in a centralized fashion,
	and this performance drop is impacted by parameters such as
	network size, communication topology and data partitioning.
	\\%
	We identify the changing consensus distance between devices as a key parameter to explain the gap between centralized and decentralized training.
	We show in theory that
	when the training consensus distance is lower than a critical quantity,
	decentralized training converges as fast as the centralized counterpart.
	We empirically validate that the relation between generalization performance
	and consensus distance is consistent with this theoretical observation.
	Our empirical insights
	allow the principled design of better decentralized training schemes that mitigate the performance drop.
	To this end, we provide practical training guidelines and exemplify its effectiveness on the data-center setup as the important first step.
\end{abstract}

\section{Introduction}

The impressive successes of machine learning, witnessed in the last decade, have been accompanied by a steady increase in the size, complexity, and computational requirements of training systems.
In response to these challenges,
distributed training algorithms (i.e.\ data-parallel large mini-batch SGD)
have been developed for the use in data-centers~\citep{goyal2017accurate,you2018imagenet,shallue2018measuring}.
These state-of-the-art (SOTA) training systems rely on the All-Reduce communication primitive
to perform exact averaging on the local mini-batch gradients computed on different subsets of the data,
for the later synchronized model update.
However, exact averaging with All-Reduce is sensitive to the communication hardware of the training system,
causing the bottleneck in efficient deep learning training.
To address this issue, decentralized training
has become an indispensable training paradigm for efficient large scale training in data-centers~\citep{assran2019stochastic},
alongside its orthogonal benefits on preserving users' privacy for edge AI~\citep{bellet2018personalized,kairouz2019advances}.

Decentralized SGD (D-SGD) implementations %
trade off the exactness of the averaging provided by All-Reduce, with more efficient, but inexact, communication over sparse typologies.
However, this often
results in a severe drop in the training and/or test performance
(i.e.\ generalization gap),
even after hyper-parameter fine-tuning \citep[see our Table~\ref{tab:understanding_the_impact_of_communication_topologies_on_dl} as well as Tables 1--3 in][]{assran2019stochastic}.
This phenomenon is poorly understood even in relatively straightforward i.i.d.\ data distribution scenarios (i.e.\ the data-center case),
to which very few works are dedicated (in fact none of them provide insights into the generalization performance).
\begin{table}[!h]
	\vspace{-0.5em}
	\centering
	\caption{\small
		\textbf{Significant generalization gap for decentralized training} on a sparse ring topology
		(ResNet-20 on CIFAR-10 with $n\!\in\!\{16,32, 64\}$ workers).
		Decentralized SGD (D-SGD) communicates model parameters through the gossip averaging.
		Test top-1 accuracies averaged over three seeds with fine-tuned learning rates.\looseness=-1
	}
	\vspace{-1.em}
	\resizebox{.4\textwidth}{!}{%
		\begin{tabular}{cccc}
			\toprule
			     & AllReduce (complete) & D-SGD (ring)     \\ \midrule
			n=16 & $92.91 \pm 0.12$     & $92.40 \pm 0.10$ \\
			n=32 & $92.82 \pm 0.27$     & $91.81 \pm 0.09$ \\
			n=64 & $92.71 \pm 0.11$     & $89.58 \pm 0.20$ \\
			\bottomrule
		\end{tabular}%
	}
	\vspace{-0.5em}
	\label{tab:understanding_the_impact_of_communication_topologies_on_dl}
\end{table}
In this work,
we investigate the trade-off between
the train/test performance and the exactness of the averaging, measured in terms of consensus distance,
i.e.\ the average discrepancy between each node and the mean of model parameters over all machines.
We identify this consensus distance as the key parameter that captures the joint effect of decentralization.

While one might suspect that a smaller consensus distance would improve performance in any case, we identify several interesting phenomena.
(i) We identify a \emph{diminishing return} phenomenon:\
if the consensus distance stays below a critical value (critical consensus distance),
decreasing the consensus distance further does not yield any additional performance gains.
For the main interests of this work, deep learning training, we (ii) identify the pivotal initial training phase
where the critical consensus distance matters and the training consensus distance heavily influences the final training and generalization performance,
and (iii) large consensus distance in later training phases can even be beneficial.

Our findings have far-reaching consequences for practice:\
By (iv) using consensus control as a principled tool to find, adaptively during training, the appropriate
trade-off between targeted generalization performance and affordable communication resources, it is possible to exploit the efficiency benefits of decentralized methods without sacrificing quality.
While our numerical study,
on Computer Vision (CV) tasks (CIFAR-10 and ImageNet-32)
as well as Natural Language Processing (NLP) tasks (transformer models for machine translation),
mainly focuses on the data-center setting with homogeneous nodes,
our findings also apply to decentralized training over time-varying topologies and the more difficult heterogeneous setting alike.

\section{Related Work}
\subsection{Decentralized Learning}
For general decentralized optimization,
common algorithms are either gradient-based methods with gossip averaging steps~\citep{kempe2003gossip,xiao2004fast,boyd2006randomized},
or problem-structure dependent methods, such as primal-dual methods~\citep{hong2017prox,sun2019distributed}.
In this work, we focus on non-convex decentralized deep learning problems
and only consider gradient-based methods with gossip averaging---%
methods that do not support stochastic gradients (not suitable for deep learning)
are omitted for the discussion.

The convergence rate of gossip averaging towards the consensus among the nodes can be expressed in terms of the (expected) spectral gap
of the mixing matrix.
\citet{lian2017can} combine SGD with gossip averaging for deep learning and show that the leading term in the convergence rate
$\smash{\cO\big(\frac{1}{n\varepsilon^2}\big)}$ %
is consistent with the convergence of the centralized mini-batch SGD~\citep{dekel2012optimal}
and the spectral gap only affects the asymptotically smaller terms.
Similar results have been observed very recently
for related schemes~\citep{scaman2017optimal,scaman2018optimal,%
	koloskova2019choco,koloskova2020decentralized,%
	koloskova2020unified,vogels2020powergossip}.
To reduce the communication overhead  (number of peer-to-peer communications),
sparse topologies have been studied recently~\citep{assran2019stochastic,Wang2019:matcha,wang2020exploring,nadiradze2020swarmsgd}.
Whilst a few recent works focus on the impact of the topology on the optimization performance~\citep{luo2019hop,neglia2020decentralized}, %
we here identify the consensus distance as a more canonical parameter that can characterize the overall effect of decentralized learning, beyond only the topology.
Through this, we are able to
provide deeper understanding of the more fine-grained impact of the evolution of the actual consensus distance on the optimization/generalization performance of deep learning.\looseness=-1
Prior works propose to
either perform a constant number of gossip steps every round~\citep{tsianos2016efficient,scaman2017optimal,Jiang2017:collaborative,%
	Jiang2018:consensus,Sharma2019:repeated} to increase the averaging quality,
or choose carefully tuned learning rates~\citep{Tsitsiklis1985:gossip,Nedic2009:distributedsubgrad,%
	Duchi2012:distributeddualaveragig,Yuan2016:decentralized} to improve the convergence.
However, these works do not analyze the varying effect of consensus distance in the phases of training explicitly.
In contrast, we identify the existence of a \emph{critical} consensus distance, \emph{adapt} gossip step numbers to the target distance on the fly, and provide insights into how consensus distance at different training phases impacts the decentralized deep learning.

Appendix~\ref{sec:connection_with_prior_work} further details the relationship between consensus distance and other training metrics influential to the final performance
(e.g.\ gradient diversity in~\citep{yin2018gradient, johnson2020adascale}).
Besides, we connect the insights into better generalization~\citep{lin2020dont}
with other interpretations in~\citep{izmailov2018averaging, gupta2020stochastic}.

\subsection{Critical Learning Phase in Deep Learning}
The connection between optimization and generalization of deep learning training is not fully understood.
A line of work on understanding the early phase of training
has recently emerged as a promising avenue for studying this connection.
For instance, \citet{keskar2016large, sagun2017empirical, achille2018critical, golatkar2019time, frankle2020the}
point out the existence of a ``critical phase''
for regularizing deep networks, which is decisive for the final generalization ability.
\citet{achille2018critical, jastrzebski2018on, fort2019emergent, jastrzebski2020the}
empirically demonstrate the rapid change in the local shape of the loss surface in the initial training phase.

In this work, we reach a similar conclusion for decentralized deep learning:\
we identify the importance of the initial training phase through the lens of consensus distance.

\section{Theoretical Understanding}\label{sec:theory}
In this section, we study the trade-off between training performance and the exactness of parameter averaging, and establish the notion of critical consensus distance.

For the sake of simplicity, we consider decentralized stochastic gradient descent (D-SGD) without momentum in this section, and focus on the optimization difficulty in our theoretical analysis.
Theoretically analyzing the generalization performance for deep learning is an open problem and not intended in this work.
Instead we provide extensive empirical evaluation, addressing generalization for both D-SGD with and without momentum in Section~\ref{sec:empirical_insights}. All proofs are deferred to Appendix~\ref{sec:proofs}.\looseness=-1

\subsection{Notation and Setting}
The agents are tasked to solve a sum-structured optimization problem $f \colon \R^d \to \R$ of the form
\begin{align}
	f^\star := \textstyle \min_{\xx \in \R^d} \left[ f(\xx) := \frac{1}{n} \sum_{i=1}^n f_i(\xx) \right] \,,
\end{align}
where the components $f_i\colon \R^d \rightarrow \R$ are distributed among the $n$ nodes and are given in stochastic form:
$
	f_i(\xx) := \EEb{\xi \sim \cD_i}{F_i (\xx, \xi)}
$,
where $\cD_i$ denotes the local data distribution on node $i \in [n]$. For data-center settings, where data is re-shuffled periodically among nodes, these distributions are identical, but in other scenarios there can be differences between nodes.
In D-SGD,
each agent $i \in [n]$ maintains local parameters $\smash{\xx_i^{(t)}} \in \R^d$,
and updates them as:
\begin{small}
	\begin{align}\label{eq:d-sgd}
		\textstyle
		\xx_{i}^{(t+1)} = \sum_{j=1}^n w_{ij} \left(\xx_j^{(t)} - \eta \nabla F_j (\xx_j^{(t)},\xi_j^{(t)})\right)\,, \tag{D-SGD}
	\end{align}
\end{small}%
that is, by a stochastic gradient step based on a sample $\xi_i^{(t)}\sim \cD_i$, followed by gossip averaging with neighboring nodes in the network encoded by the mixing weights $w_{ij}$.
As parameters can differ across nodes, we define $\bar{\xx} := \frac{1}{n} \sum_{i=1}^n \xx_i$ and $\mX := \left[ \xx_1, \ldots, \xx_n \right] \in \R^{d \times n}$,
and  $\bar{\mX} := \left[ \bar{\xx}, \ldots, \bar{\xx} \right] \equiv \mX \frac{1}{n} \1 \1^\top$.
\begin{assumption}[Mixing matrix]\label{a:W}
	Every sample of the (possibly randomized) mixing matrix $\mW \!=\! \{ w_{ij} \} \in \R^{n \times n}$
	is doubly stochastic and there exists a parameter $p > 0$ s.t.\
	\begin{small}
		\begin{align}
			\textstyle
			\E_\mW \norm{ \mX \mW - \bar{\mX} }_F^2  \leq (1 - p) \norm{\mX - \bar{\mX}}_F^2
			, \forall \mX \in \R^{d \times n}.
		\end{align}
	\end{small}
\end{assumption}
This assumption covers a broad variety of settings~\citep[see e.g.][]{koloskova2020unified}, such as D\nobreakdash-SGD with fixed (constant) mixing matrix with spectral gap $\rho$, with parameter $p = 1 - (1 - \rho)^2 = \Theta(\rho)$, but also for randomly chosen mixing matrices, for instance random matchings.
\begin{assumption}[$L$-smoothness]\label{a:lsmooth_nc}
	Each function $f_i(\xx) \colon \R^d \to \R$, $i \in [n]$
	is differentiable and there exists a constant $L \geq 0$ such that for each $\xx, \yy \in \R^d$:
	$
		\norm{\nabla f_i(\xx) - \nabla f_i(\yy) } \leq L \norm{\xx -\yy}\,. %
	$
\end{assumption}
\begin{assumption}[Bounded noise $\sigma$ and diversity $\zeta$]\label{a:noise}
	There exists constants $\sigma^2, \zeta^2$ s.t. $\forall \xx_1, \dots \xx_n \in \R^d$\vspace{-1mm}
	\begin{small}
		\begin{align}
			\begin{split}
				\textstyle
				&\frac{1}{n} \sum_{i = 1}^n \EE{\xi_i}{\norm{\nabla F_i(\xx_i, \xi_i) - \nabla f_i(\xx_i)}}^2_2 \leq \sigma^2 \,, \\
				\textstyle
				&\frac{1}{n} \sum_{i = 1}^n \norm{\nabla f_i(\xx_i) - \nabla f(\xx_i)}^2_2 \leq \zeta^2 \,.
			\end{split}
		\end{align}
	\end{small}
\end{assumption}%

\subsection{Decentralized Consensus Optimization}
Under the above standard assumptions in decentralized optimization, the convergence rate of \eqref{eq:d-sgd} has been shown as follows:
\begin{theorem}[\citet{koloskova2020unified}]
	\label{thm:1}
	Let $f_i$ be $L$-smooth and stepsize $\gamma \leq \gamma_{\rm max} = \cO \big( \smash{\frac{p}{L}} \big)$. Then there exists an optimal stepsize $\gamma \leq  \gamma_{\rm max}$ such that $\frac{1}{T} \sum_{t = 0}^{T-1} \E \norm{\nabla f(\bar\xx^{(t)})}_2^2 \leq \varepsilon$ for
	\begin{align*}
		T = \cO\bigg(\frac{\sigma^2}{n\varepsilon^2}  + \frac{\sqrt{p}\sigma + \zeta}{p\varepsilon^{3/2}} + \frac{1}{p\varepsilon}\bigg) \cdot L (f(\xx_0) - f^\star) \,.
	\end{align*}
\end{theorem}
In comparison, for centralized mini-batch SGD (C-SGD) we are allowed to choose a potentially much larger stepsize $\gamma_{\rm max}' = \cO \big( \frac{1}{L} \big)$, and can bound the number of iterations by $ \smash{\cO\bigl(\frac{\sigma^2}{n\varepsilon^2} + \frac{1}{\varepsilon}\bigr)}$. While asymptotically both these rates are equivalent, they differ in the low accuracy setting when $\varepsilon$ is not too small. That is, especially in the first phase of optimization where the lower order terms matter~\cite{bottou-bousquet-2008,keivan2021:critical}.

As our first theoretical contribution, we show that if the individual iterates of the agents stay sufficiently close, then D-SGD can converge as fast as C-SGD.
To measure this difference between agents, we use the \emph{consensus distance}
\begin{align*}\textstyle
	\Xi_t^2 := \frac{1}{n}\sum_{i=1}^n \big\|\bar \xx^{(t)} - \xx_i^{(t)}\big\|^2\,.
\end{align*}
\begin{proposition}[Critical Consensus Distance (CCD)]
	\label{rem:ccd}
	If the consensus distance is bounded by
	\begin{align}\label{eq:want}
		\Xi_t^2 \leq  \bigg( \frac{1}{Ln} \gamma \sigma^2 + \frac{1}{8 L^2}\norm{\nabla f(\bar \xx^{(t)})}^2 =: \Gamma_t^2 \bigg)
	\end{align}
	for all $t$, then in D-SGD we may choose larger stepsizes $\gamma \leq \gamma_{\max}' = \cO \big( \frac{1}{L} \big) $ and recover the convergence rate of C-SGD,
	that is $\smash{\cO\bigl(\frac{\sigma^2}{n\varepsilon^2} + \frac{1}{\varepsilon}\bigr)}$ \citep{dekel2012optimal, bottou2018optimization}.
	We refer to $\Gamma_t^2$ as \emph{critical consensus distance} (CCD).
\end{proposition}
Note that the CCD does not depend on the graph topology and that $\Gamma_t^2>0$,
which means that we do not need perfect consensus between agents to recover the C-SGD rate,
but we allow consensus distance $\Xi_t^2 \geq 0$ (i.e.\ the $\Xi_t^2 = 0$ $\forall t$,
as we have for centralized optimization, is sufficient but not necessary). %
In Section~\ref{sec:empirical_insights},
we empirically examine the existence of the critical consensus distance $\Xi_t^2$ in decentralized deep learning,
as we cannot compute the critical consensus distance in a closed-form (through $L$ and $\sigma^2$).

We now estimate the magnitude of the consensus distance in D-SGD and compare it to the CCD.
\begin{proposition}[Typical consensus distance]
	\label{rem:tcd}
	Let $\phi_t^2 := \frac{1}{n}\sum_{i=1}^n \big\|\nabla f_i(\xx_i^{(t)})\big\|^2$.
	Then under the assumption that $\gamma, p$ are constant, and the $\phi_t$ does not change too fast between iterations,
	i.e.\ not decreasing faster than exponentially:\ $\phi_{t}^2 \leq (1+p/4) \phi_{t+1}^2$, the consensus distance in D-SGD satisfies\vspace{-2mm}
	\begin{align}\label{eq:have}
		\Xi_t^2 = (1-p)\gamma^2 \cdot \cO\!\left(\frac{\phi_t^2}{p^2} + \frac{\sigma^2}{p}\right) \,.
	\end{align}
\end{proposition}
While these assumptions do not hold in epochs with learning rate decay, we observe in practice that during epochs of a constant learning rate the gradients indeed do not change too fast (see Figure~\ref{subfig:local_grad_norm}).
Thus these assumptions are reasonable approximations to capture the practical behavior.

\subsection{Controlling the Consensus Distance}
\label{sec:control}
We now investigate scenarios where the typical consensus distance derived in Proposition~\ref{rem:tcd} \emph{can} be smaller than the critical value (CCD). This reveals two orthogonal strategies to control the consensus distance in D-SGD.
We here assume diversity $\zeta=0$ as with i.i.d.\ training data, and that the stepsize $\gamma \leq \cO \big(\frac{1}{L}\big)$ as for C-SGD, and give a more refined discussion in Appendix~\ref{subsec:sufficient_bounds}.

\vspace{-1mm}
\paragraph{Learning rate decay (changing $\gamma$).}
We observe that when $\gamma =\cO\bigl( \frac{p}{n L} \bigr)$ then $\Xi_t^2 \leq \Gamma_t^2$ (if the noise~$\sigma$ is small, especially for $\sigma = 0$, then the weaker assumption $\gamma = \cO\bigl( \frac{p}{L} \bigr)$ is sufficient).
However, choosing too small stepsizes can impact performance in practice. In C-SGD the constraint on the stepsize is loose ($\gamma \leq \frac{1}{L}$). Yet, after sufficient learning rate decay, the desired CCD can be reached.

\vspace{-1mm}
\paragraph{More gossip iterations (changing $p$).}
We observe that when  $\smash{\frac{1}{1-p}} = \cO (1+\gamma L n)$, then $\Xi_t^2 \leq \Gamma_t^2$ (again, when the noise $\sigma$ is small, especially when $\sigma^2=0$, a weaker condition $\smash{\frac{1}{1-p}} = \cO (1+\gamma L)$  is sufficient).
Whilst designing new mixing topologies to control $p$ might not be possible due to practical constraints (fixed network, denser graphs increase latency, etc.), a simple and commonly used strategy is to use repeated gossip steps in every round.
\begin{lemma}[Repeated gossip~\citep{xiao2004fast, boyd2006randomized}]\label{lem:rep_gossip}
	Suppose $\mW=\mW_k\dots \mW_1$, for $k$ (possibly randomized) mixing matrices with parameter $p$ each. Then the mixing parameter for $\mW$ is at least $p_\mW \geq 1-(1-p)^k$. %
\end{lemma}
From this, we see that the mixing parameter can be improved exponentially when applying more gossip steps. To ensure $p_\mW \geq 1- \frac{1}{1+\gamma L n}$, at most $k \leq \frac{ \ln(1+\gamma L n)}{p} = \tilde\cO \big(\frac{1}{p}\big)$ repetitions are required.
\section{Inspecting Consensus Distance for Decentralized Training} \label{sec:empirical_insights}
Our %
analysis in Section~\ref{sec:theory} shows
that we can---at least in theory---recover the convergence behavior of C\nobreakdash-SGD by controlling the consensus distance.
Now, we direct our focus on generalization in decentralized deep learning training.
We show, empirically (not theoretically, see also Appendix~\ref{sec:optimization_vs_generalization}), that the critical consensus distance is an important metric to capture the connection between optimization and generalization in deep learning---%
e.g.\
Figure~\ref{fig:dec_phase_1_learning_curves} in Section~\ref{subsec:findings_cv_tasks} showcases that
by addressing the optimization difficulty in the critical initial training phase
(Figure~\ref{fig:ring_learning_curves_resnet20_cifar10_ring_training_losses}
and Figure~\ref{fig:ring_learning_curves_resnet20_cifar10_ring_training_accs}),
the final generalization gap can be perfectly closed
(Figure~\ref{fig:ring_learning_curves_resnet20_cifar10_ring_test_accs}, Table~\ref{tab:resnet20_cifar10_different_consensus_distances_and_phases_by_constant_on_ring} and Table~\ref{tab:resnet20_3_imagenet32_different_consensus_distances_and_phases_on_ring}).

First we introduce and justify our experimental design in Section~\ref{subsec:exp_design}.
We describe the implementation in Section~\ref{subsec:exp_setup}.
In Section~\ref{subsec:findings_cv_tasks}, we present our findings on image classification benchmark with standard SGD optimizer, which is the main focus of this work;
a preliminary study on Transformer with Adam optimizer and inverse square root learning rate schedule
can be found in Section~\ref{subsec:transformer}.

\subsection{Experiment Design: Controlled Training Phases}\label{subsec:exp_design}
\paragraph{Phase-wise training.}
Since the consensus distance evolves throughout training,
identifying its impact at every training step is infeasible.
However, as the consensus distance and critical consensus distance (CCD)
both significantly depend on the learning rate (Propositions~\ref{rem:ccd} and \ref{rem:tcd}),
we expect rather consistent observations during phases in which the learning rate is kept fixed and more drastic changes between such phases.
On CV tasks, stage-wise learning rate schedule is the common practice for SOTA distributed training as described in Section~\ref{subsec:exp_setup}:\
thus the training can be naturally divided into phases through the learning rate decay\footnote{
	The learning rate warmup is only over a very small fraction of training epochs (e.g.\ 5 out of 300 epochs on CIFAR-10).
	To simplify the analysis, we do not consider it as a separate phase.
},
in each of which training dynamics are significantly different from the others,
such as $\Xi_t$ (Figure~\ref{fig:uncontrolled_consensus_dist}), $\phi_t$ (Figure~\ref{subfig:local_grad_norm}) and $L$-smoothness (Figure~\ref{subfig:grad_lips}).
The transformer (NLP task) has no well-defined training phases due to the conventional inverse square root learning rate,
thus for the sake of simplicity, we consider the entire transformer training as one phase as a preliminary study.

\begin{figure}[!t]
	\centering
	\includegraphics[width=.40\textwidth,]{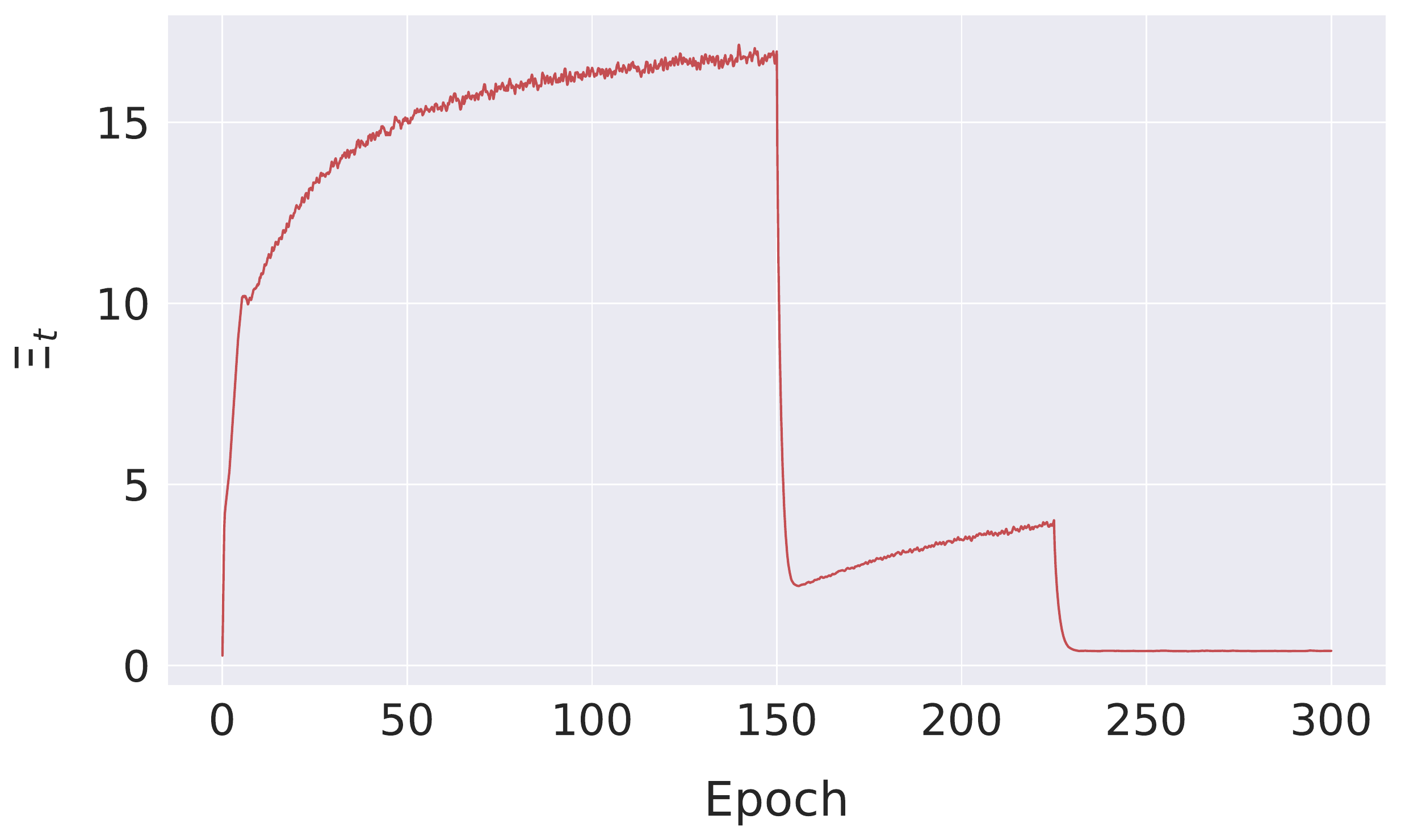}
	\vspace{-1.em}
	\caption{\small
		\textbf{Evolution of the consensus distance} $\Xi$ for ResNet-20 on CIFAR-10 ($n\!=\!32$) with ring topology.
	}
	\vspace{-1.em}
	\label{fig:uncontrolled_consensus_dist}
\end{figure}

\paragraph{Individual phase investigation.}
In order to eliminate the coupling of effects from other phases,
in each experiment we place only one phase under consensus distance control (the control refers to perform multiple gossip steps as in Section~\ref{sec:control} to reach certain distance targets),
while performing exact averaging (All-Reduce for all nodes)
on model parameters for the other unstudied phases.
We demonstrate in Table~\ref{tab:resnet20_cifar10_phase_interference_on_ring} of Section~\ref{subsec:findings_cv_tasks} that the decentralization impacts on different phases are rather orthogonal, which justifies our design of examining one phase at a time.
\\
For the ease of presentation,
the term ``phase-$x$'' refers to a training phase between ${(x\!-\!1)}$-th and $x$-th learning rate decay.
The notation ``dec-phase-$x$'' indicates that
only in ``phase-$x$'' the model is trained with a decentralized communication topology,
while for other phases we perform All-Reduce on model parameters.
We compare the result of each individually decentralized phase with that of All-Reduce centralized training (on all training phases),
so as to identify when (which phase) and how decentralized training influences final generalization performance.

\subsection{Experimental Setup}\label{subsec:exp_setup}
\paragraph{Datasets and models.} We empirically study the decentralized training behavior on the following two tasks, on convolutional neural networks and transformer architectures:
(1) Image Classification for CIFAR-10~\citep{krizhevsky2009learning}
and ImageNet-32 (i.e.\ image resolution of $32$,~\citealp{chrabaszcz2017downsampled}),
with the standard data augmentation and preprocessing scheme~\citep{he2016deep};
and (2) Neural Machine Translation for the Multi30k dataset~\citep{elliott2016multi30k}.
For Image Classification,
ResNet-20~\citep{he2016deep} with different widths are used on
CIFAR (default width of $1$) and ImageNet-32 (width factor of $3$).\footnote{
	It takes $\sim7$h to for 1 round of standard ImageNet-32 training with $n\!=\!16$ V100 on a ring, and the cost increases to  $12$h for our consensus distance controlled experiments.
	Experiments on datasets of larger scales are beyond our computation budget.
}
For Neural Machine Translation,
a down-scaled transformer architecture (by $2$ w.r.t.\ the base model in~\citealp{vaswani2017attention}) is used.
Weight initialization schemes follow~\citet{goyal2017accurate,he2015delving} and~\citet{vaswani2017attention} respectively.
Unless mentioned otherwise, our experiments are repeated over three random seeds.
\vspace{-2mm}

\begin{figure*}[!t]
	\vspace{-0.5em}
	\centering
	\subfigure[Training loss.]{
		\includegraphics[width=.315\textwidth,]{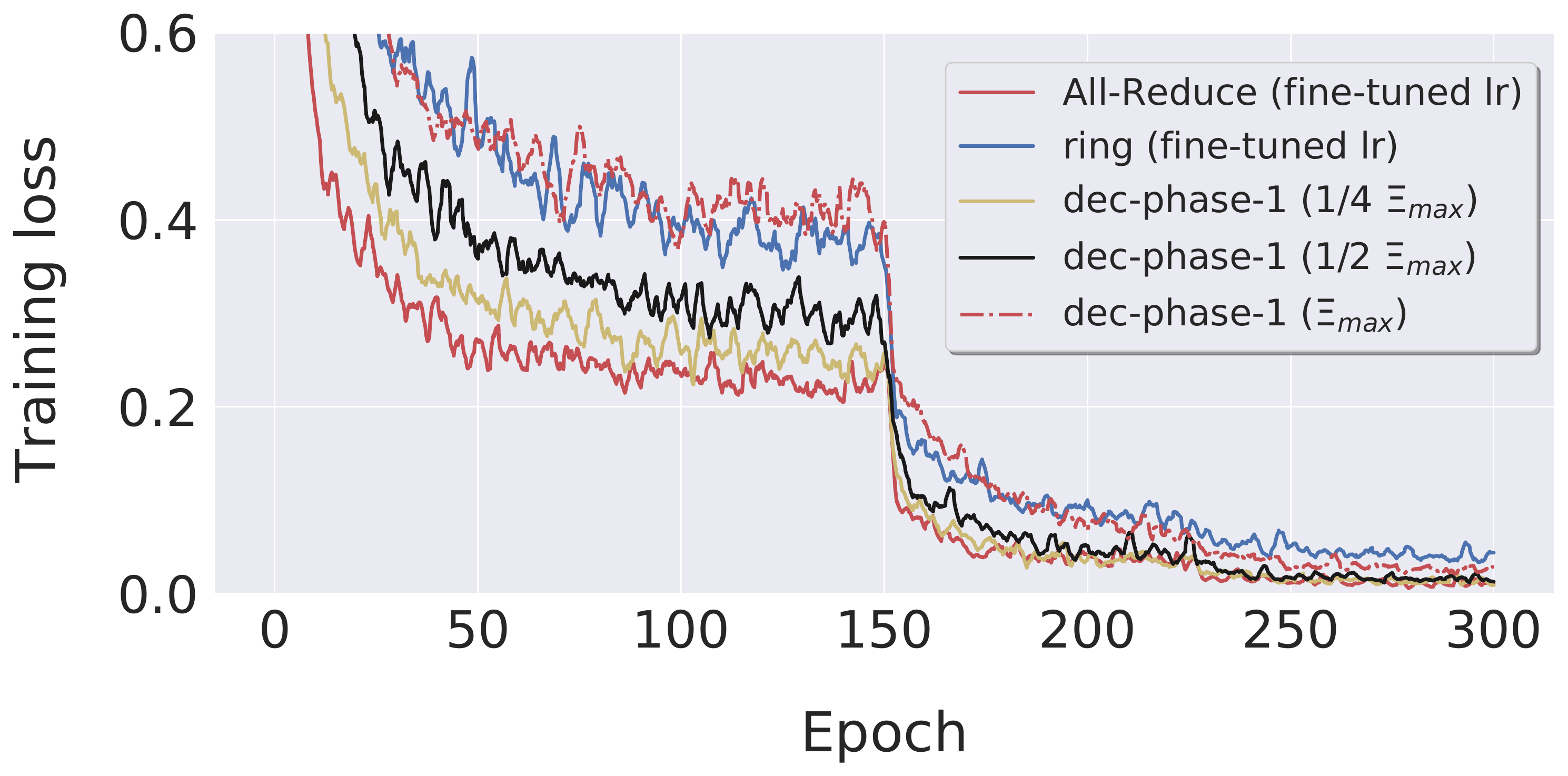}
		\label{fig:ring_learning_curves_resnet20_cifar10_ring_training_losses}
	}
	\hfill
	\subfigure[Training top-1 accuracy.]{
		\includegraphics[width=.315\textwidth,]{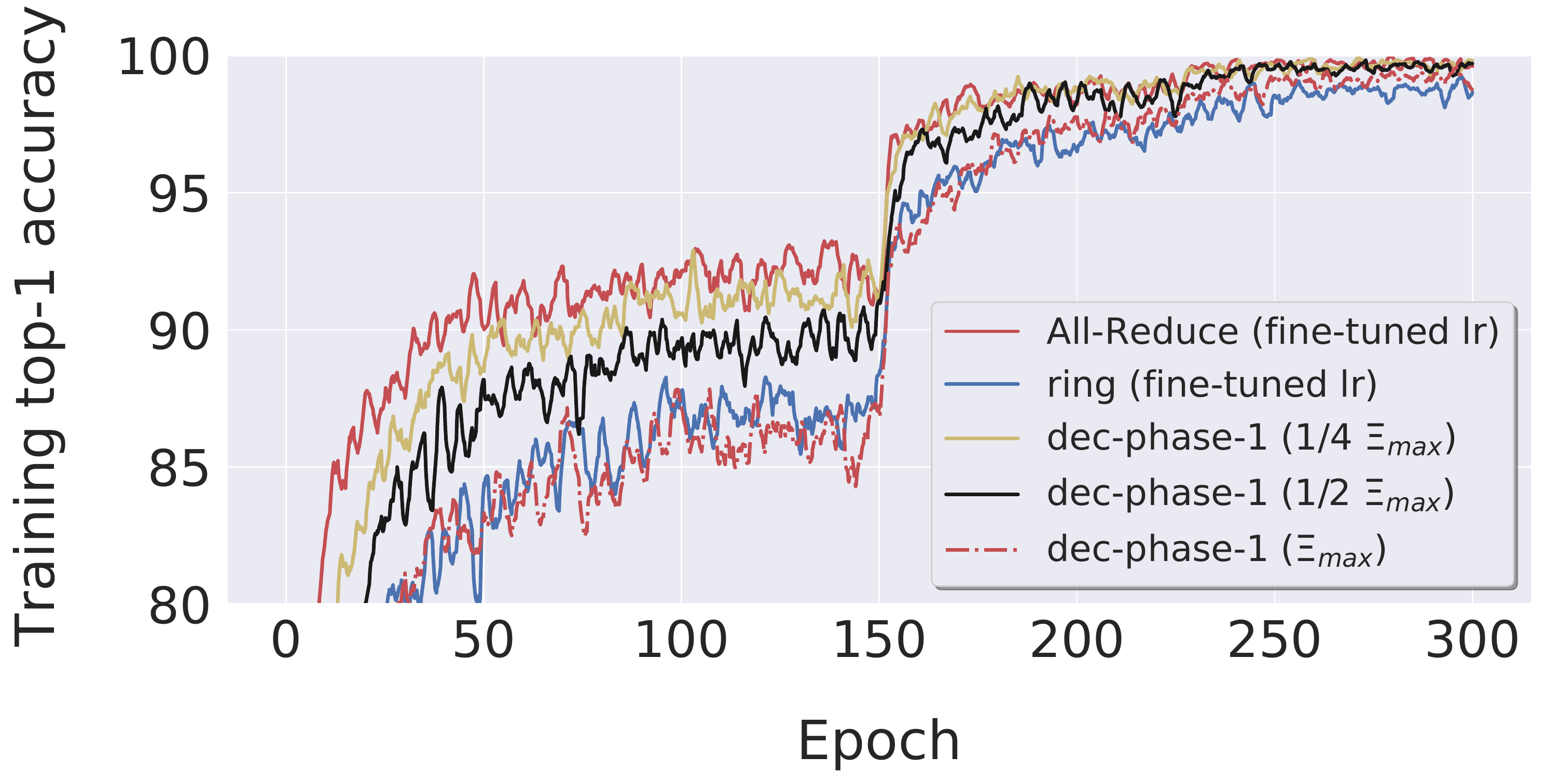}
		\label{fig:ring_learning_curves_resnet20_cifar10_ring_training_accs}
	}
	\hfill
	\subfigure[Test top-1 accuracy.]{
		\includegraphics[width=.315\textwidth,]{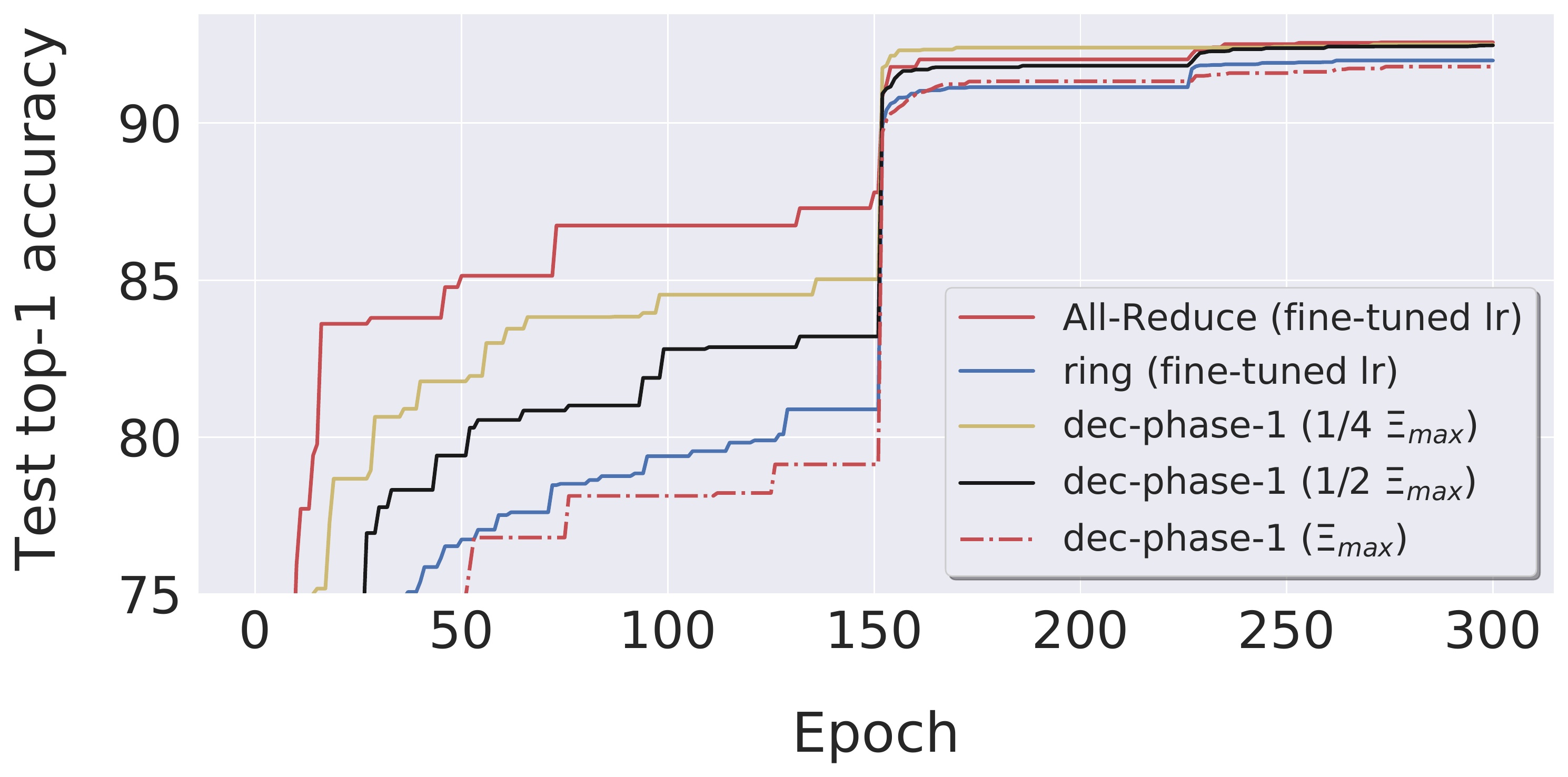}
		\label{fig:ring_learning_curves_resnet20_cifar10_ring_test_accs}
	}
	\vspace{-1em}
	\caption{\small
		Learning curves for ResNet-20 on CIFAR-10 ($n\!=\!32$).
		We compare fine-tuned normal (w/o control) decentralized training (i.e.\ ``ring'') with dec-phase-1 on different target consensus distances.
	}
	\label{fig:dec_phase_1_learning_curves}
\end{figure*}

\begin{table*}[!t]
	\centering
	\caption{\small
		\textbf{The impact of consensus distance of different phases on generalization performance} (test top-1 accuracy)
		of training ResNet-20 on CIFAR-10 on ring.
		The All-Reduce performance for $n\!=\!32$ and $n\!=\!64$
		are $92.82 \pm 0.27$ and $92.71 \pm 0.11$ respectively.
		The fine-tuned normal (w/o control) decentralized training performance for $n\!=\!32$ and $n\!=\!64$ are $91.74 \pm 0.15$ and $89.87 \pm 0.12$ respectively.\looseness=-1
	}
	\vspace{-2mm}
	\label{tab:resnet20_cifar10_different_consensus_distances_and_phases_by_constant_on_ring}
	\resizebox{1.\textwidth}{!}{%
		\huge
		\begin{tabular}{cccc|ccc|ccc}
			\toprule
			\multirow{2}{*}{\diagbox{\# nodes}{target $\Xi$}}  & \multicolumn{3}{c|}{dec-phase-1}            & \multicolumn{3}{c|}{dec-phase-2}                          & \multicolumn{3}{c}{dec-phase-3}                      \\ \cmidrule(lr){2-10}
			     & $\Ximax$         & 1/2 $\Ximax$     & 1/4 $\Ximax$     & $\Ximax$         & 1/2 $\Ximax$     & 1/4 $\Ximax$      & $\Ximax$         & 1/2 $\Ximax$     & 1/4 $\Ximax$     \\ \midrule
			n=32 & $91.78 \pm 0.35$ & $92.36 \pm 0.21$ & $92.74 \pm 0.10$ & $93.04 \pm 0.01$ & $92.99 \pm 0.30$ & $92.87 \pm 0.11$  & $92.60 \pm 0.00$ & $92.82 \pm 0.21$ & $92.85 \pm 0.24$ \\
			n=64 & $90.31 \pm 0.12$ & $92.18 \pm 0.07$ & $92.45 \pm 0.17$ & $93.14 \pm 0.04$ & $92.94 \pm 0.10$ & $92.79 \pm 0.07 $ & $92.23 \pm 0.12$ & $92.50 \pm 0.09$ & $92.60 \pm 0.10$ \\
			\bottomrule
		\end{tabular}%
	}
\end{table*}

\begin{table*}[!t]
	\centering
	\caption{\small
		\textbf{The impact of different consensus distances on generalization for different phases}
		of training ResNet-20-3 on ImageNet-32 on ring.
		The centralized baseline performances for $n\!=\!16$ and $n\!=\!32$ are $51.74 \pm 0.06$ and $51.98 \pm 0.37$ respectively,
		while those of decentralized training (on a fixed ring) are $51.04 \pm 0.06$ and $50.17 \pm 0.04$.
		The reported test top-1 accuracies are over two seeds.
	}
	\vspace{-2mm}
	\label{tab:resnet20_3_imagenet32_different_consensus_distances_and_phases_on_ring}
	\resizebox{1.\textwidth}{!}{%
		\huge
		\begin{tabular}{cccc|ccc|ccc|ccc}
			\toprule
			\multirow{2}{*}{\diagbox{\# nodes}{target $\Xi$}} & \multicolumn{3}{c|}{dec-phase-1}            & \multicolumn{3}{c|}{dec-phase-2}                          & \multicolumn{3}{c|}{dec-phase-3}             & \multicolumn{3}{c}{dec-phase-4}         \\ \cmidrule(lr){2-13}
			     & $\Ximax$         & 1/2 $\Ximax$     & 1/4 $\Ximax$     & $\Ximax$         & 1/2 $\Ximax$     & 1/4 $\Ximax$     & $\Ximax$         & 1/2 $\Ximax$     & 1/4 $\Ximax$     & $\Ximax$         & 1/2 $\Ximax$     & 1/4 $\Ximax$     \\ \midrule
			n=16 & $51.22 \pm 0.08$ & $51.79 \pm 0.10$ & $51.71 \pm 0.03$ & $51.59 \pm 0.02$ & $51.67 \pm 0.01$ & $51.65 \pm 0.13$ & $51.80 \pm 0.10$ & $51.81 \pm 0.13$ & $51.81 \pm 0.04$ & $51.72 \pm 0.02$ & $51.76 \pm 0.01$ & $51.74 \pm 0.06$ \\
			n=32 & $50.76 \pm 0.18$ & $51.27 \pm 0.07$ & $51.60 \pm 0.21$ & $51.39 \pm 0.07$ & $51.59 \pm 0.04$ & $51.66 \pm 0.12$ & $51.79 \pm 0.06$ & $51.73 \pm 0.10$ & $51.77 \pm 0.10$ & $51.70 \pm 0.02$ & $51.71 \pm 0.02$ & $51.70 \pm 0.02$ \\
			\bottomrule
		\end{tabular}%
	}
\end{table*}

\paragraph{Training schemes.}
We use mini-batch SGD with a Nesterov momentum %
of $0.9$ without dampening for image classification task
(we confirm our findings in Section~\ref{subsec:findings_cv_tasks} for SGD without momentum),
and Adam is used for neural machine translation task.
Unless mentioned otherwise we use the optimal learning rate (lr)
from centralized training for our decentralized experiments\footnote{
	We find that fine-tuning the learning rate for decentralized experiments does not change our conclusions.
	I.e.., no significant difference can be found for the curves at phase-1 for ``ring (fine-tuned lr)'' and ``dec-phase-1 ($\Ximax$)''
	in Figure~\ref{fig:ring_learning_curves_resnet20_cifar10_ring_training_losses} and~\ref{fig:ring_learning_curves_resnet20_cifar10_ring_training_accs}.
	We have similar observations in Table~\ref{tab:resnet20_cifar10_phase1_finetuned_learning_rates_on_ring}
	after the sufficient learning rate tuning on phase-1.
}
in order to observe the impact of \textit{decentralization} on normal \textit{centralized} training.

\begin{itemize}[nosep,leftmargin=12pt]
	\item For image classification experiments,
	      unless mentioned otherwise, the models are trained for $300$ and $90$ epochs for CIFAR-10 and ImageNet-32 respectively;
	      the local mini-batch size are set to $32$ and $64$.
	      By default, all experiments follow the SOTA learning rate scheme in the distributed deep learning literature~\citep{goyal2017accurate, he2019bag}
	      with learning rate scaling and warmup scheme.
	      The learning rate is always gradually warmed up from a relatively small value (i.e.\ $0.1$) for the first $5$ epochs.
	      Besides, the learning rate will be divided by $10$
	      when the model has accessed specified fractions of the total number of training samples~\citep{he2016deep};
	      we use $\{ \frac{1}{2}, \frac{3}{4} \}$ and $\{ \frac{1}{3}, \frac{2}{3}, \frac{8}{9} \}$ for CIFAR and ImageNet respectively.
	      All results in tables are test top-1 accuracy.

	\item For experiments on neural machine translation,
	      we use standard inverse square root learning rate schedule~\citep{vaswani2017attention}
	      with local mini-batch size $64$.
	      The warm-up step is set to $4000$ for the mini-batch size of $64$
	      and is linearly scaled down by the global mini-batch size.\looseness=-1
\end{itemize}

\paragraph{Consensus distance control.}
For consensus control, we adopt the ``more gossip iterations'' strategy introduced in Section~\ref{sec:control}.
That is, we perform multiple gossip steps (if needed) until reaching the desired target consensus distance value.
Two metrics are considered to set the consensus distance target value
during the specified training phase:
\begin{itemize}[nosep,leftmargin=12pt]
	\item constant target distance (main approach\footnote{
		      We use this one primarily since we can directly regulate the magnitude of consensus distance.
		      In experiments, $\text{target}\, \Xi = \Ximax$ refers to the normal (i.e.\ uncontrolled) decentralized training.
	      }):
	      the target consensus distance $\Xi$ for a phase is the \emph{maximum consensus distance} $\Ximax$
	      of the \emph{current phase} in normal (uncontrolled) decentralized training, multiplied by a factor.
	      For a given topology, the smaller the factor, the tighter the consensus.

	\item adaptive target distance (in Appendix~\ref{subsec:optimal_consensus_for_better_generalization}):
	      the target consensus distance $\Xi$ for the current step is the
	      averaged local gradient norm $\phiavg_t$ scaled by a factor.
	      For stability,
	      we use the exponentially moving averaged value $\phiema_t$ of $\phiavg_t$ (accumulated during the corresponding phase).
\end{itemize}

We use a ring as the main decentralized communication topology,
as it is a particularly hard instance with a small spectral gap (cf.\ Table~\ref{tab:spectral_gap_node_degree}) which allows us to study
a wide range of target consensus distances by modifying the number of gossip steps (in appendix we show consistent findings on time varying exponential topology in Table~\ref{tab:resnet20_cifar10_phase_interference_on_undirected_exponential_graph} and~\ref{tab:resnet20_cifar10_phase2_on_undirected_time_varying_exponential_graph}).%
.\looseness=-1

\subsection{Findings on Computer Vision Tasks} %
\label{subsec:findings_cv_tasks}
In this section we present our empirical findings and provide insights into how consensus distance at different phases
impacts the training generalization for CV tasks (i.e.\ CIFAR-10, Imagenet-32).

\paragraph{Critical consensus distance exists in the initial training phase---consensus distance below this critical threshold ensures good optimization and generalization.}
In the initial training phase, both training and generalization performance are heavily impacted by the consensus distance
(``dec-phase-1'' in Figure~\ref{fig:dec_phase_1_learning_curves} and Table~\ref{tab:resnet20_cifar10_different_consensus_distances_and_phases_by_constant_on_ring}).
A smaller consensus distance in the early phase results in considerably faster optimization (training loss) and higher generalization performance (test accuracy),
and these advantages persist over the entire training.

When the consensus distance is larger (i.e.\ 1/2 $\Ximax$ for CIFAR-10),
the optimization (training performance) can eventually catch up with the centralized convergence
(c.f.\ Figure~\ref{fig:ring_learning_curves_resnet20_cifar10_ring_training_losses} and~\ref{fig:ring_learning_curves_resnet20_cifar10_ring_training_accs})
but a considerable generalization gap still remains ($92.36$ v.s.\ $92.82$ for the setup in Figure~\ref{fig:dec_phase_1_learning_curves})
as shown in Table~\ref{tab:resnet20_cifar10_different_consensus_distances_and_phases_by_constant_on_ring}.
A consistent pattern can be found in ImageNet-32 experiments\footnote{
	1/2 $\Ximax$ has already been tight enough to recover the centralized performance
	for ImageNet-32 ($n\!=\!32$),
	while a significant performance drop can be observed between $\Ximax$ and 1/2 $\Ximax$.
}, as shown in Table~\ref{tab:resnet20_3_imagenet32_different_consensus_distances_and_phases_on_ring}.
These observations to some extent are consistent with the insights of the critical learning phase described
in~\citep{golatkar2019time,jastrzebski2020the,frankle2020the} for centralized training,
where it is argued that the initial learning phase is crucial for the final generalization.

Notably, perfect consensus distance is not required to recover the centralized training performance.
For instance, 1/4\nobreakspace$\Ximax$ is sufficient in CIFAR-10 experiments to approach the optimal centralized training performance in both optimization and \emph{generalization} at the end.
Smaller distances (e.g.\ 1/8 $\Ximax$, 1/16 $\Ximax$) do not bring significant gain ($92.77$ and $92.72$ respectively in Table~\ref{tab:elaborated_resnet20_cifar10_different_consensus_distances_and_phases_by_constant_on_ring}).
The performance saturates (c.f.\ $92.74$ for 1/4 $\Ximax$) with significantly increased communication overhead
(e.g.\ Figure~\ref{fig:resnet20_cifar10_consensus_distance_vs_comm_rounds} of Appendix~\ref{subsec:consensus_averaging_problem}).
This confirms that our analysis and discovery in Section~\ref{sec:theory} are sensible in the initial training phase:
\textit{there exists a critical consensus distance for the training, below which the impact of decentralization is negligible.}

\paragraph{A non-negligible consensus distance at middle phases can improve generalization over centralized training.}
Surprisingly, it is not always the case that the generalization performance improves with a shrinking consensus distance.
We observe that at the phase right after the initial training plateaus (e.g.\ phase-2 for CIFAR-10, phase-3 for Imagenet-32),
a non-negligible consensus distance\footnote{
	Table~\ref{tab:resnet20_cifar10_phase2_on_undirected_time_varying_exponential_graph}
	of Appendix~\ref{subsec:optimal_consensus_for_better_generalization}
	shows that there exists optimal consensus distance at middle phases,
	beyond which the gain in generalization (brought by noise injection) starts to diminish.
}
actually boosts the generalization performance over the centralized training which
has been deemed optimal.
In CIFAR-10 dec-phase-2 experiments (Table~\ref{tab:resnet20_cifar10_different_consensus_distances_and_phases_by_constant_on_ring}),
the generalization performance increases monotonically with the evaluated consensus distance
and is consistently superior to that of the centralized training (e.g.\ $93.04$, $92.99$, $92.87$ over $92.82$ for $n\!=\!32$).
Analogous observation can be obtained in Imagenet-32 dec-phase-3 experiments (Table~\ref{tab:resnet20_3_imagenet32_different_consensus_distances_and_phases_on_ring}).

\begin{table*}[!t]
	\centering
	\caption{\small
		\textbf{The impact of consensus distance on generalization performance with vanilla SGD (without momentum)} (test top-1 accuracy)
		of training ResNet-20 on CIFAR-10 on ring.
		The All-Reduce performance for $n\!=\!32$ and $n\!=\!64$
		are $90.64 \pm 0.19$ and $90.58 \pm 0.26$ respectively.
		The fine-tuned normal (w/o control) decentralized training performance for $n\!=\!32$ and $n\!=\!64$ are $90.30 \pm 0.14$ and $88.92 \pm 0.23$ respectively.
		We repeat experiments for $n\!=\!32$ for 3 seeds and $n\!=\!64$ for 2 seeds.
	}
	\vspace{-2mm}
	\label{tab:resnet20_cifar10_ring_distance_control_without_momentum}
	\resizebox{.85\textwidth}{!}{%
		\huge
		\begin{tabular}{cccc|ccc}
			\toprule
			\multirow{2}{*}{\diagbox{\# nodes}{target $\Xi$}}  & \multicolumn{3}{c|}{dec-phase-1}            & \multicolumn{3}{c}{dec-phase-2}  \\  \cmidrule(lr){2-7}
			       & $\Ximax$         & $1/2 \Ximax$     & $1/4 \Ximax$     & $\Ximax$         & $1/2 \Ximax$     & $1/4 \Ximax$     \\ \hline
			$n=32$ & $90.51 \pm 0.05$ & $90.74 \pm 0.14$ & $90.88 \pm 0.37$ & $90.64 \pm 0.18$ & $90.55 \pm 0.19$ & $90.57 \pm 0.17$ \\
			$n=64$ & $88.80 \pm 0.03$ & $89.89 \pm 0.03$ & $90.43 \pm 0.05$ & $90.63 \pm 0.37$ & $90.46 \pm 0.15$ & $90.63 \pm 0.25$ \\
			\bottomrule
		\end{tabular}%
	}
	\vspace{-0.5em}
\end{table*}

This coincides with the observations firstly introduced in post-local SGD~\citep{lin2020dont},
where for better generalization, consensus distance is created among local machines by less frequent model parameter synchronization (All-Reduce)
in late training phases (e.g.\ phase-2, phase-3 for CIFAR).
Thus non-negligible consensus distance at middle phases can be viewed as a means of injecting proper noise as argued in~\citep{lin2020dont},
which reduces communication cost and in the meanwhile benefits generalization.

\paragraph{At the last phase of training, the consensus distance only marginally impacts the generalization performance.}
Similar to the initial training phase, the final convergence phase seems to favor small consensus distances in CIFAR-10 experiments.
However, its impact is less prominent in comparison:
for dec-phase-3, performance of a smaller consensus distance (1/4 $\Ximax$) is only $0.25\%$ and $0.37\%$ higher
than that of $\Ximax$ for $n\!=\!32, 64$ respectively (Table~\ref{tab:resnet20_cifar10_different_consensus_distances_and_phases_by_constant_on_ring}).
In Imagenet-32 experiments, dec-phase-3 performance is not even affected by changes in consensus.

\paragraph{Quality propagation across phases.}
Our previous experiments only consider a single phase of decentralized training. We now evaluate the lasting impact of consensus across the sequence of multiple phases.
In Table~\ref{tab:resnet20_cifar10_phase_interference_on_ring},
we control the consensus distance for both phase-1 and phase-2 when training on CIFAR-10.
Our previous findings hold when we view each controlled phase separately.
For instance, when we apply 1/2 $\Ximax$ consensus control to phase-2 (the middle column in Table~\ref{tab:resnet20_cifar10_phase_interference_on_ring}),
we can still observe that a smaller consensus distance in phase-1 results in a higher performance as in our previous finding.
Hence our previous findings are valid in more general cases of decentralized training.

\begin{table}[!b]
	\caption{\small
		\textbf{Quality propagation across training phases with different consensus distances} on ResNet-20 for CIFAR-10 (Ring with $n\!=\!32$).
		In phase-1 and phase-2, the model parameters reach inexact consensus of different target consensus distance $\Xi$,
		while phase-3 performs All-Reduce on model parameters.
	}
	\label{tab:resnet20_cifar10_phase_interference_on_ring}
	\vspace{-1em}
	\centering
	\resizebox{.5\textwidth}{!}{%
		\begin{tabular}{cccc}
			\toprule
			\diagbox{phase-1}{phase-2} & $\Ximax$         & 1/2 $\Ximax$     & 1/4 $\Ximax$     \\ \midrule
			1/2 $\Ximax$               & $92.48 \pm 0.19$ & $92.46 \pm 0.11$ & $92.31 \pm 0.23$ \\
			1/4 $\Ximax$               & $92.73 \pm 0.11$ & $92.66 \pm 0.08$ & $92.69 \pm 0.19$ \\
			1/8 $\Ximax$               & $93.10 \pm 0.22$ & $92.88 \pm 0.15$ & $92.91 \pm 0.06$ \\
			\bottomrule
		\end{tabular}%
	}
	\vspace{-0.5em}
\end{table}

\begin{table}[!t]
	\centering
	\caption{\small
		\textbf{The impact of different numbers of training epochs (at phase-1)} on generalization,
		for training ResNet-20 on CIFAR-10 (dec-phase-1 with $n\!=\!32$).
		The number of epochs at phase-1 is chosen from $\{ 150, 200, 250 \}$,
		while the other training setting is identical to that of dec-phase-1 in Table~\ref{tab:resnet20_cifar10_different_consensus_distances_and_phases_by_constant_on_ring}.
	}
	\label{tab:resnet20_cifar10_impact_of_n_training_epochs_for_phase1_on_ring}
	\vspace{-2mm}
	\resizebox{.5\textwidth}{!}{%
		\begin{tabular}{cccc}
			\toprule
			\multirow{2}{*}{ target $\Xi$} & \multicolumn{3}{c}{training epochs at phase-1}  \\ \cmidrule{2-4}
			             & 150              & 200              & 250              \\ \midrule
			$\Ximax$     & $91.78 \pm 0.35$ & $91.91 \pm 0.19$ & $92.04 \pm 0.14$ \\
			1/2 $\Ximax$ & $92.36 \pm 0.21$ & $92.55 \pm 0.07$ & $92.67 \pm 0.13$ \\
			1/4 $\Ximax$ & $92.74 \pm 0.10$ & $92.91 \pm 0.15$ & $92.84 \pm 0.20$ \\
			\bottomrule
		\end{tabular}%
	}
	\vspace{-1em}
\end{table}

\paragraph{Longer training cannot close the generalization gap caused by large consensus distances in the initial training phase.}
As discussed above, large consensus distances in the initial phase can result in significant generalization loss.
Table~\ref{tab:resnet20_cifar10_impact_of_n_training_epochs_for_phase1_on_ring}
investigates whether a prolonged training on the initial phase can address this difficulty:\
we prolong the phase-1 for CIFAR-10 with a range of consensus distances
and leave the other training phases centralized.
We can observe that although longer training is beneficial for each consensus distance,
it cannot recover the generalization gap resulting from large consensus distance.
For instance, the maximum gain (among all evaluated cases) of increasing the epoch number from 150 to 250 is $0.31\%$ at 1/2 $\Ximax$,
which is lower than the average gain (around $0.6\%$) of merely reducing the consensus distance from $\Ximax$ to 1/2 $\Ximax$.
Table~\ref{tab:resnet20_cifar10_ring_prolonged_phase_2_3} in Appendix~\ref{sec:more_understanding}
evaluates cases where dec-phase-2 and dec-phase-3 are prolonged.
We find longer training in these two phases brings about negligible performance gain.

\begin{figure*}[!t]
	\vspace{-0.5em}
	\centering
	\subfigure[\small
		Different target $\Xi$s.
	]{
		\includegraphics[width=.315\textwidth,]{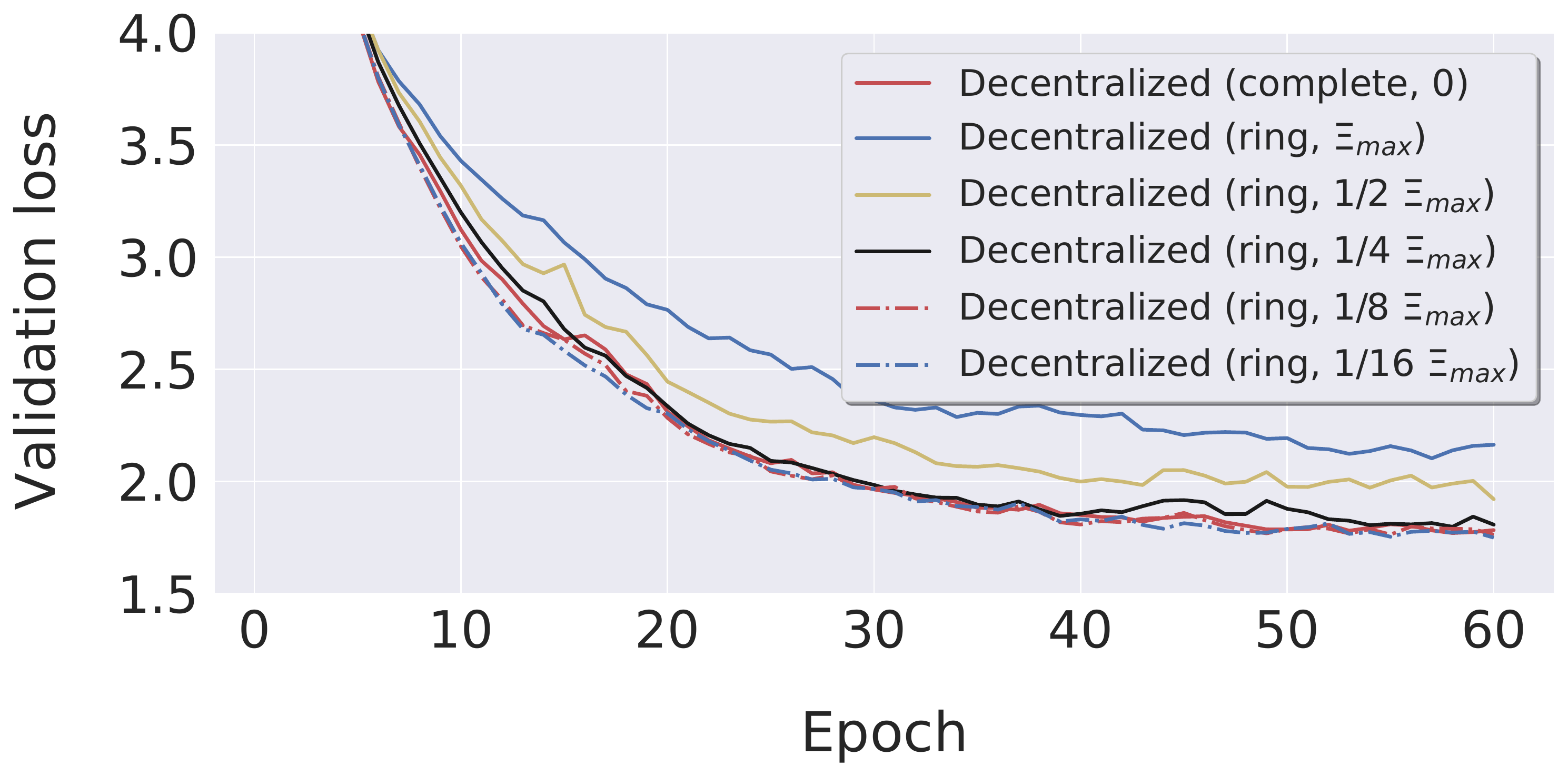}
		\label{fig:transformer_multi30k_controlled_first_60epochs}
	}
	\hfill
	\subfigure[\small
		Decentralized baseline.
	]{
		\includegraphics[width=.315\textwidth,]{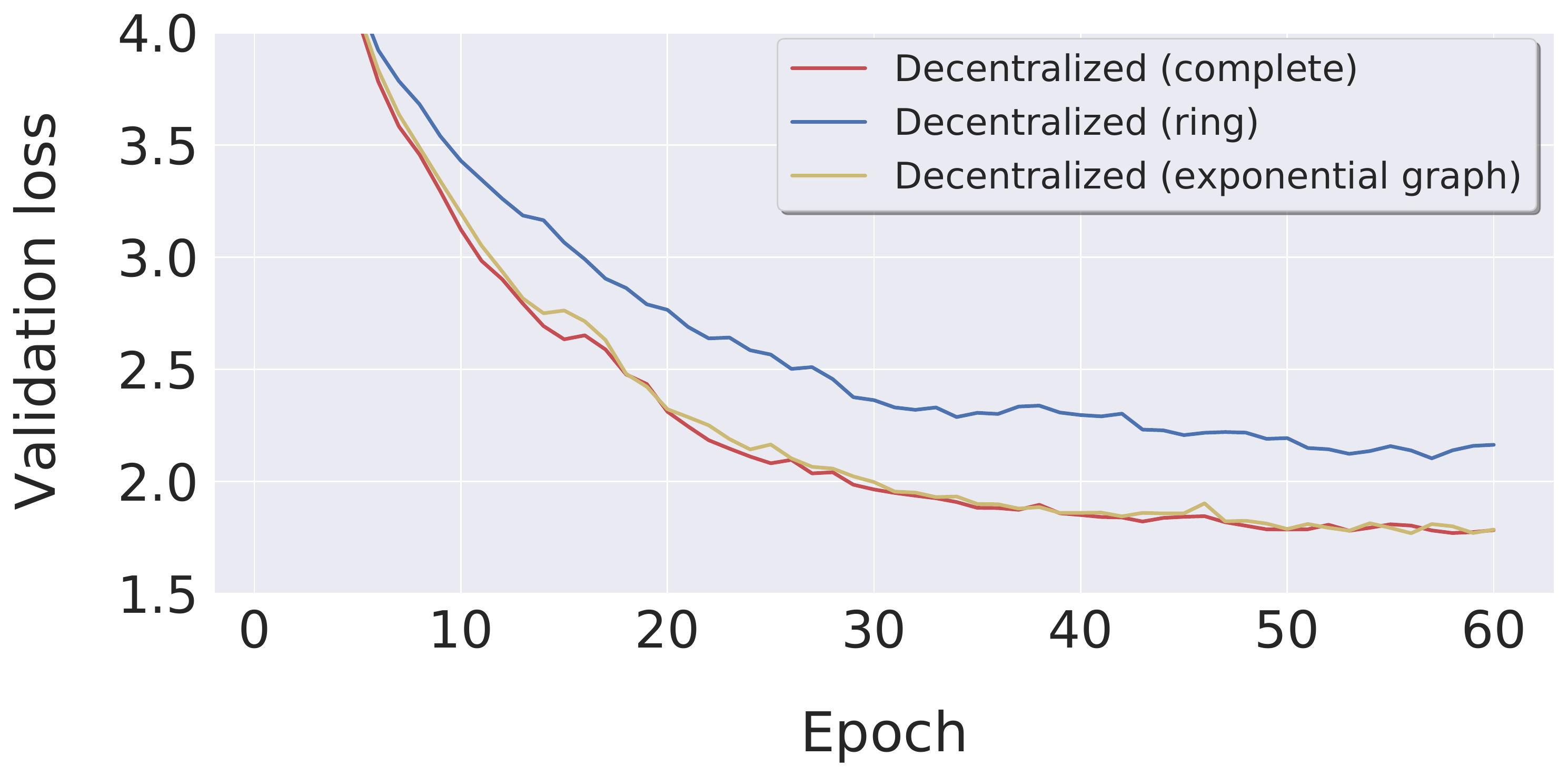}
		\label{fig:transformer_multi30k_baseline_decentralized_only}
	}
	\hfill
	\subfigure[\small
		Consensus distance $\Xi$.
	]{
		\includegraphics[width=.315\textwidth,]{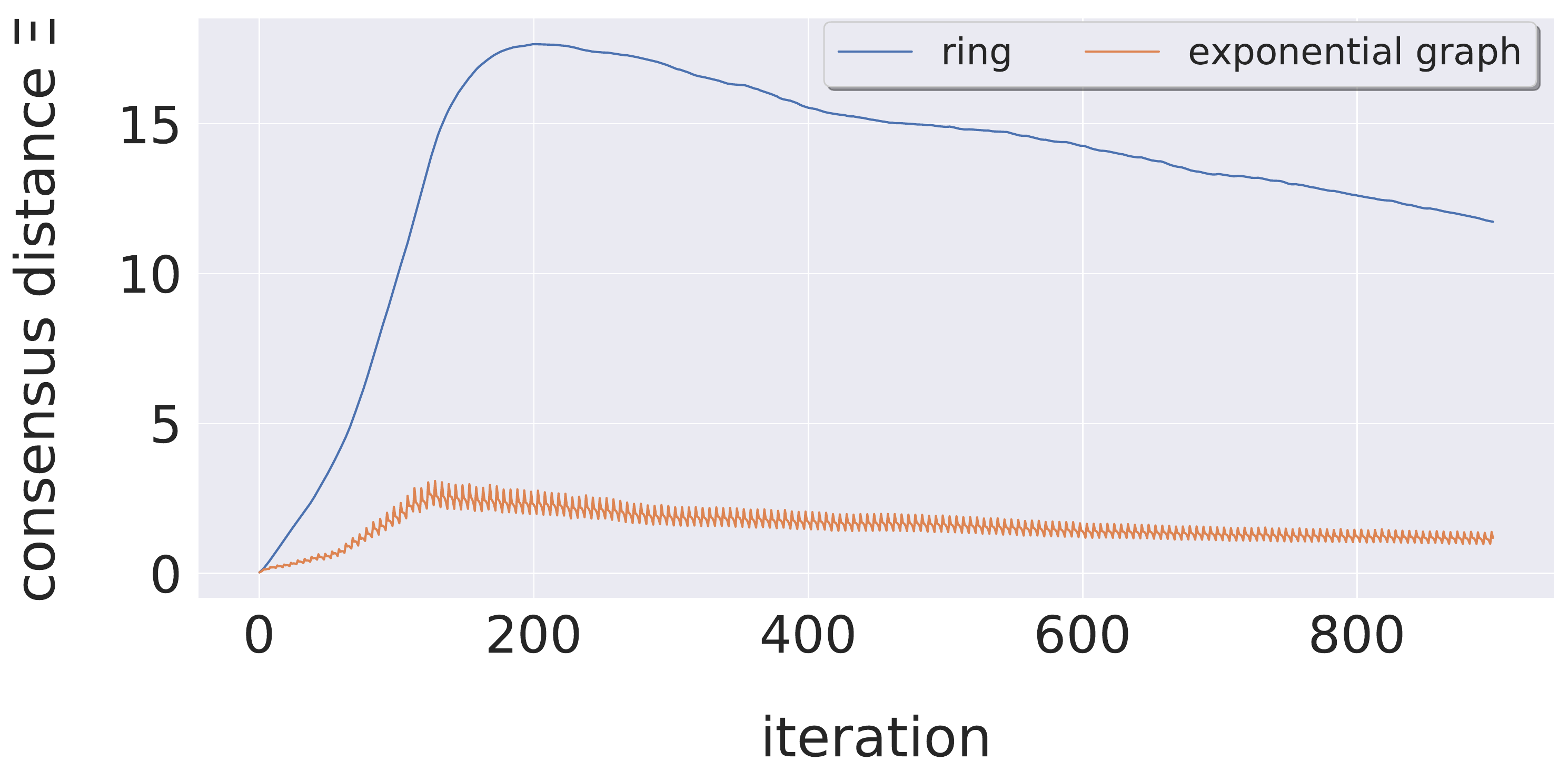}
		\label{fig:transformer_multi30k_consensus_distance_for_ring_and_exponential_graph}
	}
	\vspace{-1em}
	\caption{\small
		Learning curves for training Transformer on Multi30k ($n\!=\!32$).
	}
	\label{fig:preliminary_transformer}
\end{figure*}

\begin{table*}[!t]
	\centering
	\caption{\small
		\textbf{The importance of phase-1} for training ResNet-20 on CIFAR-10 ($n\!=\!32$),
		in terms of \textbf{(1) target consensus distance} and \textbf{(2) the number of training epochs}.
		In phase-2 and phase-3, we perform \emph{decentralized} training (w/o consensus distance control).
	}
	\vspace{-1em}
	\label{tab:resnet20_cifar10_n32_phase_1_importance_with_uncontrolled_other_phases}
	\resizebox{.8\textwidth}{!}{%
		\begin{tabular}{cccccc}
			\toprule
			\diagbox{\# of epochs}{target $\Xi$} & $\Ximax$         & 1/2 $\Ximax$     & 1/4 $\Ximax$     & 1/8 $\Ximax$     & 0 $\Ximax$       \\ \midrule
			150                                  & $91.74 \pm 0.15$ & $92.31 \pm 0.12$ & $92.81 \pm 0.22$ & $92.91 \pm 0.15$ & $92.94 \pm 0.07$ \\
			200                                  & $91.81 \pm 0.22$ & $92.88 \pm 0.20$ & $93.00 \pm 0.18$ & $93.01 \pm 0.10$ & $92.90 \pm 0.17$ \\
			250                                  & $92.09 \pm 0.23$ & $92.74 \pm 0.11$ & $93.15 \pm 0.26$ & $92.99 \pm 0.24$ & $93.31 \pm 0.06$ \\
			\bottomrule
		\end{tabular}%
	}
	\vspace{-0.5em}
\end{table*}

\paragraph{Consistent findings on decentralized SGD without momentum.}
To validate the coherence between our theory and experiments,
we perform similar consensus distance control experiments on vanilla SGD optimizer
(i.e.\ without momentum)
for dec-phase-1 and dec-phase-2 on CIFAR-10.
The patterns illustrated in Table~\ref{tab:resnet20_cifar10_ring_distance_control_without_momentum}
are consistent with our previous observations in
Table~\ref{tab:resnet20_cifar10_different_consensus_distances_and_phases_by_constant_on_ring}
and Table~\ref{tab:resnet20_3_imagenet32_different_consensus_distances_and_phases_on_ring},
supporting the claim on the relation between consensus distance and generalization performance
(which stands regardless of the use of momentum).

\subsection{
	Preliminary study on training transformer models %
}
\label{subsec:transformer}
\paragraph{The critical consensus distance also exists in NLP tasks.}
Figure~\ref{fig:transformer_multi30k_controlled_first_60epochs} demonstrates that
1/4 $\Ximax$ target control on a ring is sufficient to recover the centralized training performance.
Besides, the target consensus distance in this case can be reached by exponential graph (and thus target test performance,
as shown in Figure~\ref{fig:transformer_multi30k_baseline_decentralized_only} and~\ref{fig:transformer_multi30k_consensus_distance_for_ring_and_exponential_graph}).
These justify the importance of designing an efficient communication topology/scheme in practice
so as to effectively reach the CCD.\looseness=-1

\section{Impact on Practice}~\label{sec:impact_on_practice}
~\vspace{-0.8cm}
\paragraph{Practical guidelines: prioritizing the initial training phase.}
Apart from effectiveness (generalization/test performance), efficiency (time) stands as the other crucial goal in deep learning,
and thus
how to allocate communication resource over the training becomes a relevant question.

\begin{figure}[!h]
	\centering
	\includegraphics[width=.45\textwidth,]{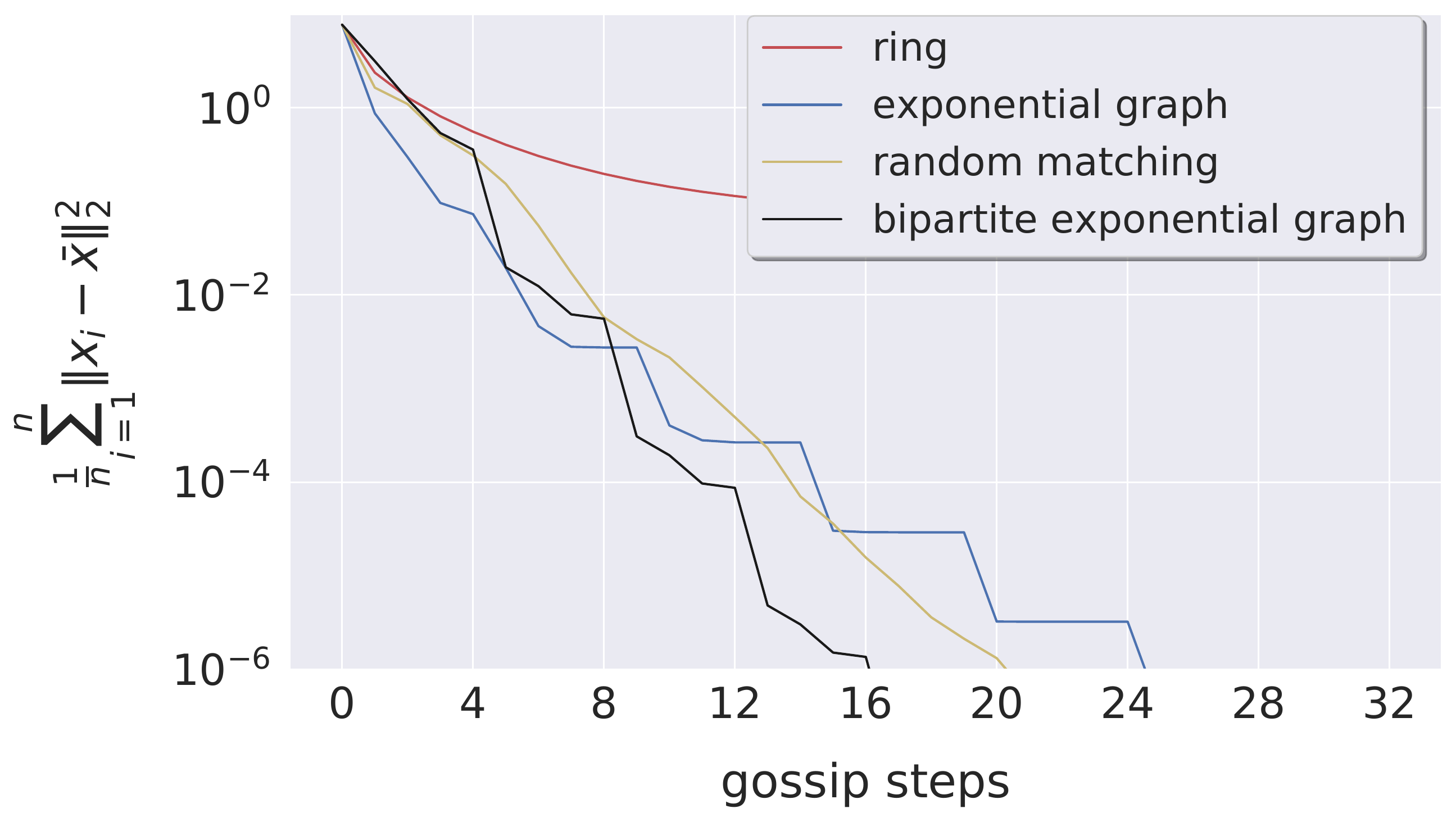}
	\vspace{-1.em}
	\caption{\small
		\textbf{Consensus distance evolution against the number of gossip steps} on different topologies ($n\!=\!32$).
		The initial $x_{i}$'s are sampled uniformly from $[0, 10]$.
		Results on different topology scales are deferred to Appendix~\ref{subsec:consensus_averaging_problem}.
	}
	\vspace{-0.5em}
	\label{fig:understanding_the_importance_of_communication_topologies}
\end{figure}

As indicated by our first empirical finding (and theory in Section~\ref{sec:theory}),
the initial training phase bears the greatest importance over all other training phases;
therefore the communication expenditure should be concentrated on the initial phase
to maintain a consensus distance lower than the CCD.
We suggest a list of communication topologies with superior spectral properties,
i.e.\ exponential graph~\citep{assran2019stochastic}
and random matching~\citep{nadiradze2020swarmsgd} in Figure~\ref{fig:understanding_the_importance_of_communication_topologies}
(the definition of the topology is detailed in Appendix~\ref{subsec:consensus_averaging_problem}),
which can be utilized to achieve fast convergence in gossip averaging.\looseness=-1

The late training phases should be less prioritized for communication resources,
due to the generalization benefits from a reasonable consensus distance in the middle phases.
Providing a rigorous way to quantify the optimal consensus distance is non-trivial,
and is left as future work.

In Table~\ref{tab:resnet20_cifar10_n32_phase_1_importance_with_uncontrolled_other_phases}
we show that the above-mentioned guideline is practically feasible:\
as long as the quality of the initial phase is ensured,
we can afford to slacken the consensus control for later phases, in particular the middle phase.
For instance, when the number of epochs is 150, a consensus control of 1/4 $\Ximax$ in the initial phase
with uncontrolled middle and final phase
is adequate to recover the centralized training performance ($92.81$ v.s.\ $92.82$).
Note that here the noise injection from the uncontrolled middle phase also contributes positively to the performance.
Table~\ref{tab:resnet20_cifar10_phase_interference_on_undirected_exponential_graph} in Appendix~\ref{subsec:optimal_consensus_for_better_generalization}
additionally justifies the applicability of applying this guideline on exponential graphs.

\paragraph{Practical implementation of Consensus Control in Data-Centers.}
Computing the exact consensus distance
requires the average of all model parameters in $\R^{d}$, which is prohibitively expensive (All-Reduce).
We propose therefore to use the following efficient estimator \vspace{-2mm}
\begin{small}
	\begin{align*}
		\Theta_t^2 & := \frac{1}{n}\sum_{i = 1}^n \theta_i^{(t)} &
		           & \text{with}                                 &
		\theta_i^{(t)} &:=  \bigg\| \sum_{j=1}^n w_{ij} \xx_j^{(t)} - \xx_i^{(t)} \bigg\|_2^2 \,,
	\end{align*}
\end{small}%
\par\vspace{-0.4cm}
instead (in Lemma~\ref{lem:upper_bound_CD} we prove that $\Xi_t \leq \smash{\frac{2}{p}} \Theta_t$ is an upper-bound of consensus distance and thus a valid control parameter, see also Section~\ref{sec:eff-impl-exp} for numerical validation).
The values $\smash{\theta_i^{(t)}} \in \R$ can be computed \emph{locally} on each node when updating the parameters at negligible cost (compared to gradient computations), and computing $\Theta_t$ requires only averaging of scalars. While this can be implemented efficiently in data-centers (the cost of averaging these scalar values is negligible compared to averaging high-dimensional parameter vectors in the gossip steps), this might not be efficient over arbitrary decentralized network. \\%
Table~\ref{tab:local_estimate_control_dec_phase_1} and~\ref{tab:local_estimate_control_guidline} in Appendix~\ref{sec:eff-impl-exp} show the feasibility of integrating the control of $\Theta_t$ with our practical guidelines for efficient training in data-centers,
which serves as a strong starting point for designing decentralized training algorithms with a desired balance between communication cost and training performance.

\section{Conclusion}
In this work, we theoretically identify the consensus distance as an essential factor for decentralized training.
We show the existence of a critical consensus distance, below which the consensus distance does not hinder optimization.
Our deep learning experiments validate our theoretical findings and extend them to the generalization performance.
Based on these insights, we propose practical guidelines %
for favorable generalization performance with low communication expenses, on arbitrary communication networks.\\%
While we focused in this work on data-center training with iid data as an important first step, consensus control may be of even greater importance in non-i.i.d.\ scenarios~\citep[such as in][]{Hsieh2020:nonIID}.

\section*{Acknowledgements}
We acknowledge funding from a Google Focused Research Award, Facebook, and European Horizon 2020 FET Proactive Project DIGIPREDICT.

\begin{small}
\bibliography{icml2021}
\bibliographystyle{configuration/icml2021}
\end{small}


\clearpage
\appendix

\onecolumn
{
	\hypersetup{linkcolor=black}
	\parskip=0em
	\renewcommand{\contentsname}{Contents of Appendix}
	\tableofcontents
	\addtocontents{toc}{\protect\setcounter{tocdepth}{3}}
}

\section{Efficient Implementation of Consensus Control for Data-Center Training}\label{sec:efficient_consensus_control}

In our theoretical and experimental investigations in Sections~\ref{sec:theory} and \ref{sec:empirical_insights}, in order to understand the effect of decentralization on the final performance, we focused on the controlling the consensus distance $\Xi_t^2$. This quantity was inspired by theoretical analysis and naturally measures the decentralization level. In practice, in order to control the consensus distance, one need to know the exact value of it at every iteration. Computing the exact value of the $\Xi_t^2$ requires all-to-all communications of the parameter vectors $\xx_i$, which is costly and would cancel all the practical benefits of using decentralized algorithms.

In this section we give a more practical way to control the consensus distance without compromising the final test performance.
We mainly focus on the data-center training scenario,
the most common use case of large-scale deep learning training for both academic and industry.
Though the prior arts use All-Reduce to compute the exactly averaged model parameters,
recent trends show promising faster training results by using decentralized training with gossip averaging~\citep{assran2019stochastic,koloskova2020decentralized},
especially for the highly over-parameterized SOTA neural networks with large number of model parameters.

\subsection{Theoretical Justification}

We upper bound the consensus distance $\Xi_t^2$ with a quantity that is efficiently computable in our scenario and control only this quantity. This quantity additionally requires the centralized all-reduce applied only to one dimensional numbers, that is fast to perform, and it utilizes the information available after decentralized communications step of parameters $\xx_i$ performed by the \eqref{eq:d-sgd} algorithm. For simplicity, in this section we only analyze the case of the fixed topology, i.e. mixing matrix $W$ is constant. Our result can be generalized for the randomized mixing matrix (Assumption~\ref{a:W}) and in later sections we provide the proofs under the general Assumption~\ref{a:W}.
\begin{lemma}[Upper bound on the consensus distance]\label{lem:upper_bound_CD}
	Let $\Theta_t^2 := \frac{1}{n}\sum_{i=1}^n \norm{\sum_{j=1}^n w_{ij} \xx_j^{(t)} - \xx_i^{(t)}}_2^2 = \frac{1}{n} \sum_{i=1}^n \theta_i^{(t)}$, where $w_{ij}$ are the weights of the (fixed) mixing matrix $W$. We can upper bound the consensus distance as
	\begin{align*}
		\Xi_t \leq \frac{2}{p} \Theta_t, \qquad\qquad \forall \xx_1^{(t)}, \dots, \xx_n^{(t)} \in \R^d,
	\end{align*}
	where $p$ is consensus rate of matrix $W$ (Assumption~\ref{a:W}).
\end{lemma}
To ensure small consensus distance $\Xi_t^2$ it is sufficient to make small the quantity $\Theta_t^2$. In particular by ensuring that  $\Theta_t^2 \leq \frac{p^2}{4} \Gamma_t^2$ we automatically get that CCD condition holds: $\Xi_t^2 \leq \Gamma_t^2$ (Proposition~\ref{rem:ccd}).

\paragraph{Practical way to compute $\Theta_t^2$. }
Recall that $\Theta_t^2 = \frac{1}{n} \sum_{i=1}^n \norm{\sum_{j=1}^n w_{ij} \xx_j^{(t)} - \xx_i^{(t)}}_2^2$. Each term $i$, $i \in \{1, \dots n\}$ of this sum is locally available to the node $i$ after one round of decentralized communication with mixing matrix $W$ because $w_{ij} \neq 0$ only for the neighbours $j$ of the node $i$. So each node $i$ can locally compute the norm $\norm{\sum_{j=1}^n w_{ij} \xx_j^{(t)} - \xx_i^{(t)}}_2^2$ and then obtain the average $\Theta_t^2$ using centralized all-reduce on only 1-dimensional numbers, that is much faster than averaging full vectors from $\R^d$.

\subsection*{Proof of the Lemma~\ref{lem:upper_bound_CD} }
\begin{proof}
	Using matrix notation we can re-write $\Xi_t^2 = \frac{1}{n}\norm{\mX^{(t)} - \bar \mX^{(t)}}_F^2$ and $\Theta_t^2 = \frac{1}{n} \norm{\mX^{(t)} \mW - \mX^{(t)}}_F^2$.

	Since $\bar \mX^{(t)} \mW = \bar \mX^{(t)}$ and $ \mX^{(t)} \frac{\1\1^\top}{n}  =  \bar \mX^{(t)}\frac{\1\1^\top}{n}  = \bar \mX^{(t)}$ we have that
	\begin{align*}
		\mX^{(t)} \mW - \mX^{(t)} = \left( \mX^{(t)} - \bar \mX^{(t)}\right)\left(\mW - \frac{\1\1^\top}{n} - \mI\right)
	\end{align*}
	Using Frobenius norm property \eqref{eq:frob_lowe_bound},
	\begin{align*}
		\norm{\mX^{(t)} \mW - \mX^{(t)}}_F & \geq \lambda_{\min}\left(\mW - \frac{\1\1^\top}{n} - \mI\right) \norm{ \mX^{(t)} - \bar \mX^{(t)} }_F \stackrel{\eqref{eq:eigval_lower_b}}{\geq} \frac{p}{2} \norm{ \mX^{(t)} - \bar \mX^{(t)} }_F \qedhere
	\end{align*}
\end{proof}
\subsection*{Used Inequalities}
\begin{lemma}
	For $\mA \in \R^{d\times n}, \mB \in \R^{n \times n}$, and $\mB$ symmetric
	\begin{align}\label{eq:frob_lowe_bound}
		\norm{\mA \mB}_F \geq |\lambda_{\min}(\mB)| \norm{\mA}_F,
	\end{align}
	where $\lambda_{\min}(\mB)$ is the smallest eigenvalue by the absolute value.

\end{lemma}

\begin{lemma}\label{lem:eigval_lower_b}
	Let $\mW$ be a fixed mixing matrix satisfying Assumption~\ref{a:W}. Then,
	\begin{align}\label{eq:eigval_lower_b}
		\lambda_{\min}\left(\mW - \frac{\1\1^\top}{n} - \mI\right) \geq \frac{p}{2} \,.
	\end{align}
\end{lemma}
\begin{proof}
	Let $\mU\mLambda \mU$ be SVD-decomposition of $\mW$. By Assumption~\ref{a:W}, $\mW$ is symmetric and doubly stochastic. Because of the stochastic property of $\mW$, one of the eigenvalues is equal to $1$ with the corresponding eigenvector $u_1 = \frac{1}{\sqrt{n}} \1$.

	We can represent the matrix $\frac{\1\1^\top}{n}$ and $\mI$ as
	\begin{align*}
		\frac{\1\1^\top}{n} = u_1 u_1^\top = \mU \begin{pmatrix}
			1 & 0 & \dots & 0 \\
			0 & 0 & \dots & 0 \\
			\dots\\
			0 & 0 & \dots & 0
		\end{pmatrix} \mU^\top
		  &   & \mI = \mU \mU^\top \,.
	\end{align*}
	Therefore,
	\begin{align*}
		W - \frac{\1\1^\top}{n} - \mI = \mU \left[ \mLambda - \begin{pmatrix}
				1 & 0 & \dots & 0 \\
				0 & 0 & \dots & 0 \\
				\dots\\
				0 & 0 & \dots & 0
			\end{pmatrix}  - \mI   \right]\mU^\top = \mU \begin{pmatrix}
			- 1 & 0             & \dots & 0             \\
			0   & \lambda_2 - 1 & \dots & 0             \\
			\dots\\
			0   & 0             & \dots & \lambda_n - 1
		\end{pmatrix} \mU^\top
	\end{align*}
	\begin{align*}
		\lambda_{\min}\left( \mW - \frac{\1\1^\top}{n} - \mI \right) = \min\{ 1, | 1 - \lambda_i(\mW)| \}, &   & i \geq 2
	\end{align*}
	We will prove now that every $\lambda_i(\mW), i \geq 2$ is smaller than $\sqrt{1 - p}$. Then it would follow that $\lambda_{\min}\left(\mW - \frac{\1\1^\top}{n} - \mI\right)\geq 1 - \sqrt{1 - p} \geq \frac{p}{2}$ for $0 \leq  p \leq 1$.

	Lets assume that one of the $\lambda_i(\mW)$ is greater than $\sqrt{1 - p}$. W.l.o.g. lets call this eigenvalue $\lambda_2$. Its corresponding eigenvector is $u_2$. Since the eigenvectors are orthogonal to each other and the first is $u_1 = \1$, it should hold that $u_1^2\top u_2 = \1^\top u_2 = 0$. Lets take $\mX$ such that its every column is equal to $\nicefrac{\1}{n} + u_2$. Then $\bar \mX = \frac{\1\1^\top}{n}$ and $(\mI - \frac{\1\1^\top}{n}) \mX = \mX - \bar \mX = \begin{pmatrix}
			u_2, \dots, u_2
		\end{pmatrix} =: \mU_2$ is a matrix with all columns equal to $u_2$.
	\begin{align*}
		\norm{\mW \mX -  \bar \mX }_F^2 = \norm{ \left(\mW - \frac{\1\1^\top}{n}\right)\left(\mI - \frac{\1\1^\top}{n}\right) \mX}_F^2 = \norm{\left(\mW - \frac{\1\1^\top}{n}\right) \mU_2}_F^2 = \norm{\lambda_2 \mU_2}_F^2 = \lambda_2^2 \norm{\mX -  \bar \mX}_F^2
	\end{align*}
	Since the Assumption~\ref{a:W} holds for the $\mW$ for all $\mX$, it should also hold for our chosen above $\mX$ and
	\begin{align*}
		\lambda_2^2 \norm{\mX -  \bar \mX}_F^2 = \norm{\mW \mX -  \bar \mX }_F^2 = \norm{ \left(\mW - \frac{\1\1^\top}{n}\right)\left(\mI - \frac{\1\1^\top}{n}\right) \mX}_F^2 \leq (1 - p) \norm{\left(\mI - \frac{\1\1^\top}{n}\right) \mX}_F^2
	\end{align*}
	We got a contradiction which concludes the proof.
	\qedhere

\end{proof}

\subsection{Experiments with the Efficient Consensus Control Scheme}\label{sec:eff-impl-exp}
We implement the efficient consensus control scheme to train ResNet-20 on CIFAR-10 with a ring topology ($n=32$).
We compute $\Theta$ after each gossip step as an indicator of the exact consensus distance $\Xi$.
The gossip continues until $\Theta < q \phiema $,
where $q$ is the control factor and $\phiema$ is the exponential moving average estimator
of the average norm of local gradients $\phi$.
Please refer to Section~\ref{subsec:exp_setup} for other training details.

We validate Lemma~\ref{lem:upper_bound_CD} by Figure~\ref{fig:theta_xi_fine} and Figure~\ref{fig:theta_xi_coarse}.
In Figure~\ref{fig:theta_xi_fine}, we can observe that during an arbitrary interval of the control phase
the high correlation between $\Xi$ and $\Theta$ over gossips steps.
In Figure~\ref{fig:theta_xi_coarse},
we can observe that this corrected behavior also manifests in a large span of iterations.
These observations justify our claim that the $\Theta$ can act as a decent and much more inexpensive estimator of $\Xi$.
We also plot $\phi$ over iterations in Figure~\ref{fig:ema_local_updates} to demonstrate that
the critical consensus distance $\Gamma$ stays relatively constant within each training phase.

In Table~\ref{tab:local_estimate_control_dec_phase_1},
we show the test performance of the dec-phase-1 under the control of this efficient implementation.
The pattern is consistent with the discovery in the main text.
Moreover, in Table~\ref{tab:local_estimate_control_guidline}, we follow the ``prioritizing the initial training phase'' guideline in Section~\ref{sec:impact_on_practice}.
Specifically, we control only the initial phase (phase-1) with the local estimate, while leaving the other phases uncontrolled (normal decentralized training).
We can observe that with our guideline, we can recover and surpass the centralized training baseline with only the control on the initial phase.
Therefore, combining the insights into the effect of consensus distance and this efficient implementation,
we open up the opportunities for practical decentralized training schemes
with a desired balance between communication cost and training performance.
We leave more sophisticated design for future work.

\begin{figure*}[!t]
	\vspace{-0.5em}
	\centering
	\subfigure[Consensus distance $\Xi$ (top figure) and the local estimator $\Theta$ (bottom figure) over gossip steps.]{
		\includegraphics[width=.305\textwidth,]{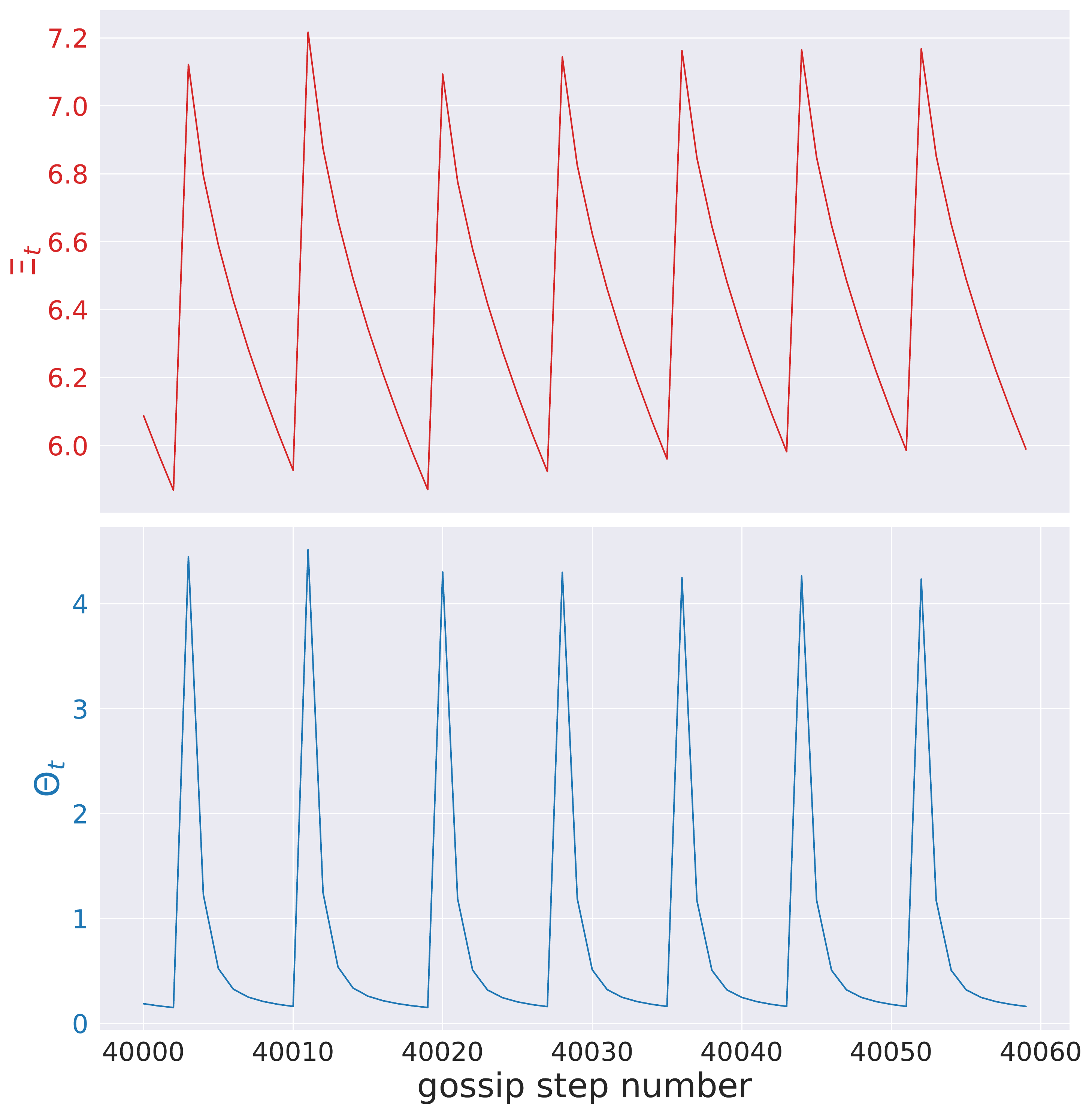}
		\label{fig:theta_xi_fine}
	}
	\hfill
	\subfigure[Consensus distance $\Xi$ (top figure) and the local estimator $\Theta$ (bottom figure) over training iterations of phase-1.]{
		\includegraphics[width=.305\textwidth,]{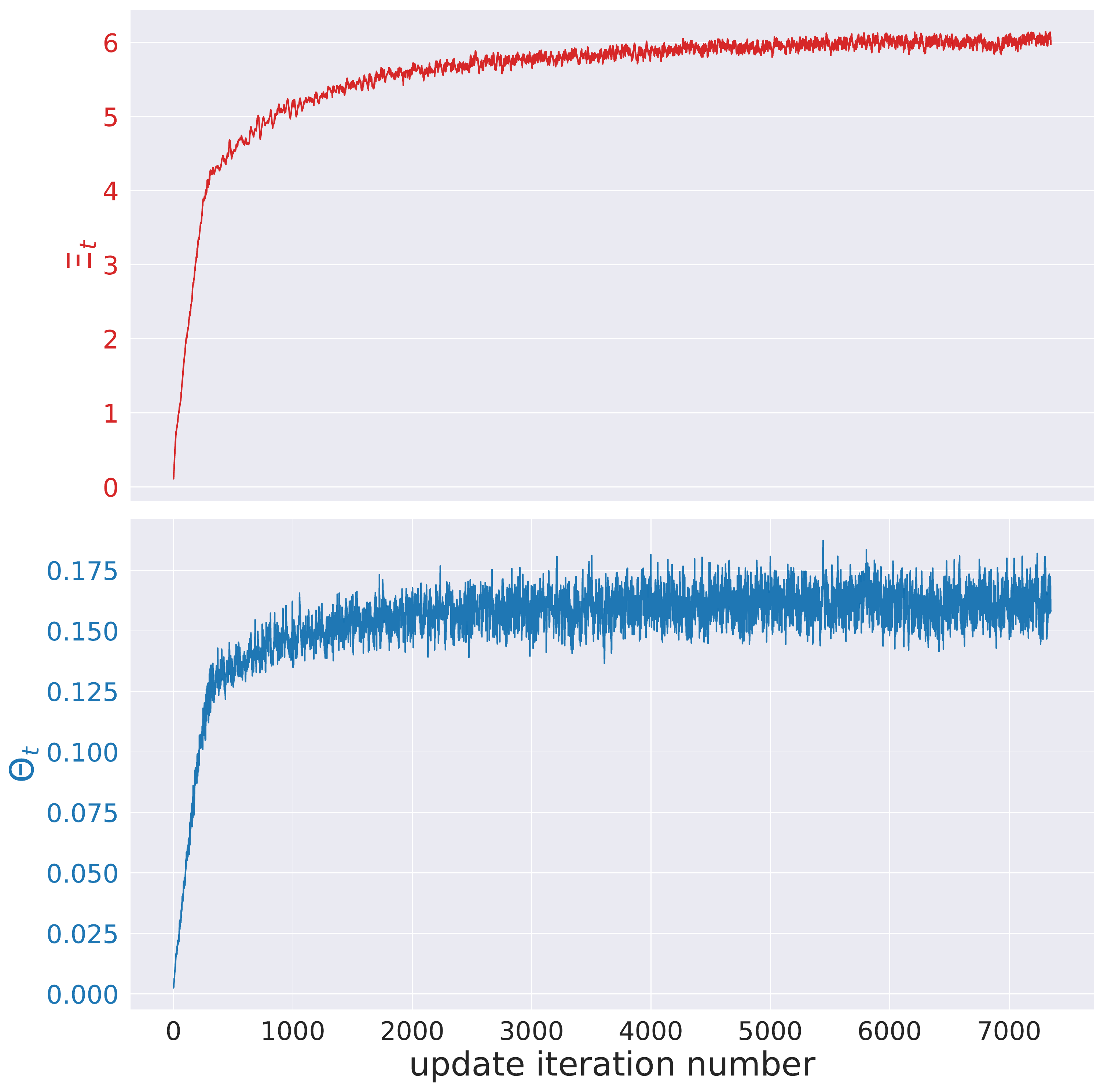}
		\label{fig:theta_xi_coarse}
	}
	\hfill
	\subfigure[The average norm of local gradients $\phi$]{
		\includegraphics[width=.305\textwidth,]{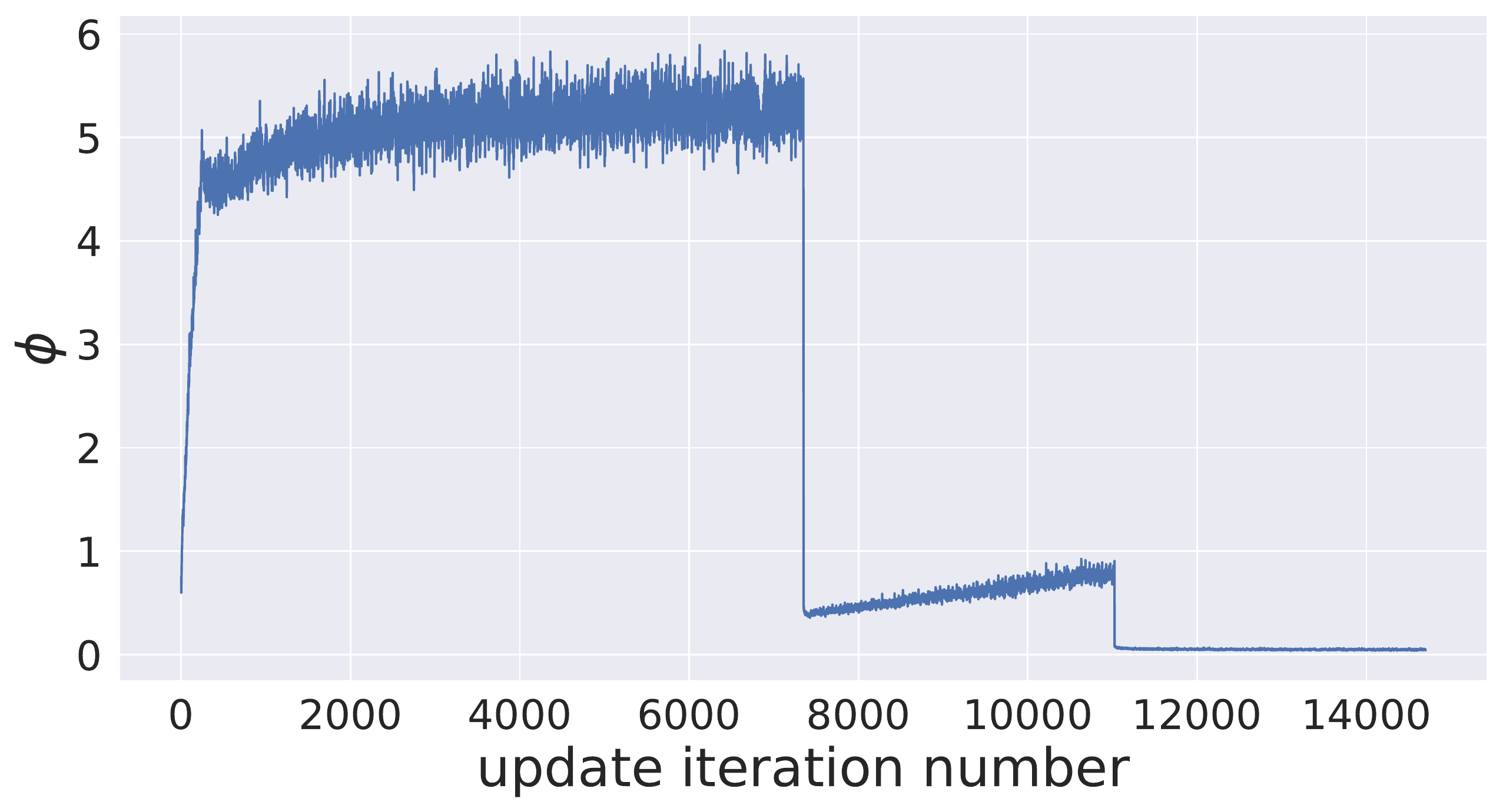}
		\label{fig:ema_local_updates}
	}
	\vspace{-1em}
	\caption{\small
		Dec-phase-1: ResNet-20 on CIFAR-10 with a ring topology ($n=32$), under the efficient implementation of the consensus control with the ratio $q = 1e{-}3$.
		To illustrate the evolution of the consensus distance, the plots are made over gossip steps.
		Note, typically several gossip steps correspond to one training iteration for consensus control.
		In Figure~\ref{fig:theta_xi_coarse}, both $\Xi$ and $\Theta$ are plotted over update iterations which correspond to the last gossip steps of all iterations (i.e.\ the troughs in Figure~\ref{fig:theta_xi_fine}); the gossip steps in Figure~\ref{fig:theta_xi_fine} correspond to iteration steps $ 4824 - 4831 $ in Figure~\ref{fig:theta_xi_coarse},
		and we only showcase an arbitrary interval in Figure~\ref{fig:theta_xi_fine} due to the consistent pattern over the entire phase 1.
	}
	\label{fig:local_estimates}
	\vspace{-0.5em}
\end{figure*}

\begin{table}[!h]
	\centering
	\caption{\small
		\textbf{
			Efficient consensus control of dec-phase-1 with local estimates} of training ResNet-20 on CIFAR-10 ($n\!=\!32$).
	}
	\vspace{-1em}
	\label{tab:local_estimate_control_dec_phase_1}
	\resizebox{.8\textwidth}{!}{%
		\begin{tabular}{c|ccccc}
			\toprule
			            & Centralized        & $q=1e{-}4$       & $q=1e{-}3$       & $q=1e{-}2$       & w/o control      \\ \midrule
			dec-phase-1 & $ 92.82 \pm 0.27 $ & $92.85 \pm 0.16$ & $92.69 \pm 0.31$ & $92.44 \pm 0.02$ & $91.78 \pm 0.35$ \\
			\bottomrule
		\end{tabular}%
	}
\end{table}

\begin{table}[!h]
	\centering
	\caption{\small
		\textbf{Efficient consensus control for data-center training---combining practical guideline with local estimates}---for training ResNet-20 on CIFAR-10 ($n\!=\!32$).
		Based on our practical guideline (Section~\ref{sec:impact_on_practice}), we control only the initial phase (phase-1) with the local estimate ($\Theta$), while leaving the other phases uncontrolled (normal decentralized training).
	}
	\vspace{-1em}
	\label{tab:local_estimate_control_guidline}
	\resizebox{.8\textwidth}{!}{%
		\begin{tabular}{c|ccccc}
			\toprule
			          & Centralized        & $q=1e{-}4$      & $q=1e{-}3$       & $q=1e{-}2$       & w/o control      \\ \midrule
			guideline & $ 92.82 \pm 0.27 $ & $93.10 \pm	0.13$ & $92.79 \pm 0.17$ & $92.64 \pm 0.14$ & $91.78 \pm 0.35$ \\
			\bottomrule
		\end{tabular}%
	}
\end{table}

\section{Related Work} \label{sec:connection}
\subsection{Connection with Prior Work} \label{sec:connection_with_prior_work}
\paragraph{Connection with gradient diversity.}
The connections between the consensus distance and gradient diversity measure~\citep{yin2018gradient,johnson2020adascale}
are not obvious and is an interesting direction for future works.
On the one hand, low gradient diversity could cause generalization degradation of decentralized methods as in the centralized case;
one the other hand, high gradient diversity increases the difficulty of reaching a low consensus distance.

\paragraph{Connection with other methods like SWA/SWAP.}
Our empirical insights bear similarity to the ones in SWA~\citep{izmailov2018averaging}, SWAP~\citep{gupta2020stochastic},
and Post-local SGD~\citep{lin2020dont}, but none of them considers decentralized deep learning.

In SWA, models are sampled from the later stages of an SGD training run:
the average of these model parameters result in a model with much higher generalization performance.
SWAP extends SWA to a parallel fashion: it uses a large batch size to train the model until close to convergence and
then switches to several individual runs with a small mini-batch size.
These individual runs serve as a way of sampling from a posterior distribution
and sampled models are averaged for better generalization performance (i.e. the idea of SWA).

Post-local SGD, SWA, SWAP, as well as the empirical insights presented in our paper,
are closely related: we first need sufficient small consensus distance to guarantee the optimization quality
(in post-local SGD, SWA, and SWAP, the consensus distance equals 0),
and thus different model averaging choices can be utilized in the later training phase for better generalization.
Considering the later training phase, our empirical observations in decentralized learning
suggest that we can improve the generalization through the simultaneous SGD with gossip averaging.
This is analogous to SWA and SWAP that sample model independently (i.e., perform SGD)
from the well-trained model and average over sampled models;
and close to Post-local SGD which performs simultaneous SGD steps with infrequent averaging.

\subsection{Discussion on ``Convergence analysis v.s.\ generalization performance''} \label{sec:optimization_vs_generalization}

\paragraph{From convergence analysis to better understand generalization.}
A line of recent research reveals the interference between initial training (optimization)~\citep{jastrzebski2020the, golatkar2019time, achille2018critical}
and the later reached local minima (generalization)~\citep{neyshabur2017implicit, lin2020dont, lin2020extrapolation, gupta2020stochastic, izmailov2018averaging, keskar2016large}:
the generalization of the deep nets cannot be studied alone via vacuous generalization bounds,
and its practical performance is contingent on the critical initial learning (optimization) phase,
which can be characterized by the conventional convergence analysis~\citep{achille2018critical, izmailov2018averaging, golatkar2019time, lin2020dont, gupta2020stochastic, jastrzebski2020the}.

This motivates us to derive the metric (i.e.\ critical consensus distance) from the convergence analysis,
for the examination of the consensus distance (on different phases) in decentralized deep learning training.
For example, (1) we identify the impact of different consensus distances
at the critical learning phase on the quality of initial optimization,
and the final generalization~\citep{jastrzebski2020the,golatkar2019time,achille2018critical,lin2020dont}
(i.e.\ our studied case of dec-phase-1),
and (2) we reveal similar observations as in~\citep{lin2020dont,lin2020extrapolation,gupta2020stochastic,izmailov2018averaging}
when the optimization is no longer a problem (our studied case of dec-phase-2),
where the existence of consensus distance can act as a form of noise injection~\citep{lin2020dont}
or sampling models from the posterior distribution~\citep{gupta2020stochastic,izmailov2018averaging} as discussed above.

\section{Theory}~\label{sec:proofs}
In this section, we prove the claims from Section~\ref{sec:theory}.

\subsection{Proof of Proposition~\ref{rem:ccd}, Critical Consensus Distance}\label{apx:ccd}
The proof of this claim follows by the following Lemma:
\begin{lemma}[\citet{koloskova2020unified}, Descent lemma for non-convex case]\label{lem:decrease_nc} Under the given assumptions, and for any stepsize $\gamma < \frac{1}{4 L}$, the iterates of D-SGD satisfy
	\begin{align*}
		\EE{t + 1}{f(\bar{\xx}^{(t + 1)})} & \leq f(\bar{\xx}^{(t)}) - \frac{\eta}{4}\norm{\nabla f(\bar{\xx}^{(t)})}_2^2 + \gamma \Xi_t^2 + \frac{L}{n} \gamma^2 \hat{\sigma}^2.
	\end{align*}
\end{lemma}
\begin{proof}
	By replacing $\Xi_t$ in the above inequality with~\eqref{eq:want}, we simplify:
	\begin{align*}
		\EE{t + 1}{f(\bar{\xx}^{(t + 1)})} & \leq f(\bar{\xx}^{(t)}) - \frac{\eta}{8}\norm{\nabla f(\bar{\xx}^{(t)})}_2^2 + \frac{2L}{n} \gamma^2 \hat{\sigma}^2.
	\end{align*}
	This inequality now matches (up to differences in the constants) the standard recursion that one can derive for C-SGD~\citep{dekel2012optimal, bottou2018optimization, stich2019unified, Stich2019:error-feedback}.
\end{proof}

\subsection{Proof of Proposition~\ref{rem:tcd}, typical consensus distance}\label{apx:tcd}
We need an auxiliary (but standard) lemma, to estimate the change of the consensus distance between iterations.
\begin{lemma}[Consensus distance] It holds
	\label{lemma:distance}
	\begin{align*}
		\Xi_{t+1}^2 \leq (1-p/2) \Xi_t^2 +  \frac{3(1-p)\gamma^2}{p} \left( \phi_t^2 + p \sigma^2\right) \,.
	\end{align*}
\end{lemma}
We give the proof of this statement shortly below. First, let us consider how this lemma allows the proof of the claim. For this, we first consider a particular special case, and assume $\phi_t \leq \phi$, for a constant $\phi$. In this case, we can easily verify by unrolling the recursion:
\begin{align*}
	\Xi_t^2 \leq \sum_{i=0}^{t-1} (1-p/2)^{i} \frac{3(1-p)\gamma^2 (\phi^2 + p \sigma^2)}{p} \leq 6 (1-p) \gamma^2 \left( \frac{\phi^2}{p^2} + \frac{\sigma^2}{p} \right)\,.
\end{align*}
Now, for the claim in the main text, we use assumption that $\phi_t$ are changing slowly, that is, not decreasing faster than exponentially: $\phi_{t}^2 \leq (1+p/4) \phi_{t+1}^2$.
With this assumption, and observing $(1-p/2)^i(1+p/4)^i \leq (1-p/4)^i$, we can unroll as before
\begin{align*}
	\Xi_t^2 & \leq \sum_{i=0}^{t-1} (1-p/2)^{i} \frac{3(1-p)\gamma^2 (\phi_{t-i-1}^2 + p \sigma^2)}{p} \\
	        & \leq \sum_{i=0}^{t-1} (1-p/4)^{i} \frac{3(1-p)\gamma^2 (\phi_{t-1}^2 + p \sigma^2)}{p}
	\leq 12 (1-p) \gamma^2 \left( \frac{\phi_{t-1}^2}{p^2} + \frac{\sigma^2}{p} \right)\,.
\end{align*}

\begin{proof}[Proof of Lemma~\ref{lemma:distance}]
	We use the following matrix notation here
	\begin{align*}
		\mX^{(t)}                      & := \left[ \xx_1^{(t)},\dots, \xx_n^{(t)}\right] \in \R^{d\times n}\,,                                     \\
		\bar{\mX}^{(t)}                & := \left[ \bar{\xx}^{(t)},\dots, \bar{\xx}^{(t)}\right] = \mX^{(t)} \frac{1}{n} \1\1^\top\,,              \\
		\nabla F(\mX^{(t)}, \xi^{(t)}) & := \left[\nabla F_1(\xx_{1}^{(t)}, \xi_1^{(t)}), \dots,  \nabla F_n(\xx_{n}^{(t)}, \xi_n^{(t)})\right]\,, \\
		\nabla f(\mX^{(t)})            & := \left[\nabla f_1(\xx_{1}^{(t)}), \dots,  \nabla f_n(\xx_{n}^{(t)})\right]\,.
	\end{align*}
	As a reminder, $\Xi_t^2 := \frac{1}{n}\sum_{i=1}^n \norm{\bar \xx^{(t)} - \xx_i^{(t)}}^2$, and $\phi_t^2 := \frac{1}{n}\sum_{i=1}^n \norm{\nabla f_i(\xx_i^{(t)})}^2$.
	\begin{align*}
		n \Xi_{t + 1}^2 & = \norm{\bar{\mX}^{(t +1)} - \mX^{(t + 1)}}_F^2 = \norm{(\mX^{(t)} - \gamma \nabla F (\mX^{(t)}, \xi^{(t)}))\left(\frac{1}{n} \1\1^\top - \mW\right) }_F^2                                      \\
		                & = \norm{(\mX^{(t)} - \gamma \nabla F (\mX^{(t)}, \xi^{(t)}))\left(\frac{1}{n} \1\1^\top - \mI \right) \left(\frac{1}{n} \1\1^\top - \mW\right) }_F^2                                            \\
		                & \leq (1 - p)\norm{(\mX^{(t)} - \gamma \nabla F (\mX^{(t)}, \xi^{(t)}))\left(\frac{1}{n} \1\1^\top - \mI \right)}_F^2                                                                            \\
		                & \leq (1 - p)\norm{(\mX^{(t)} - \gamma \nabla f (\mX^{(t)}))\left(\frac{1}{n} \1\1^\top - \mI \right)}_F^2  + (1 - p) \gamma^2 \norm{\nabla f (\mX^{(t)}) - \nabla F (\mX^{(t)}, \xi^{(t)})}_F^2 \\
		                & \leq (1 - p)(1 + \alpha) \norm{\mX^{(t)}\left(\frac{1}{n} \1\1^\top - \mI \right)}_F^2 + (1 - p) (1 + \alpha^{-1}) \gamma^2 \norm{ \nabla f (\mX^{(t)})}_F^2 + (1 - p) \gamma^2 \sigma^2 n      \\
		                & \stackrel{\alpha = \frac{p}{2}}{\leq} \left(1 - \frac{p}{2}\right) n \Xi_t^2 + \frac{3 (1 - p)}{p} \gamma^2 \norm{ \nabla f (\mX^{(t)})}_F^2 + (1 - p) \gamma^2 \sigma^2 n\,. \qedhere
	\end{align*}
\end{proof}

\subsection{Sufficient bounds to meet critical consensus distance}~\label{subsec:sufficient_bounds}
In this section, we show that the claimed bounds in Section~\ref{sec:control} are sufficient conditions to reach the CCD.

According to Proposition~\ref{rem:tcd}, there exists an absolute constant $C$, (w.l.o.g. $C \geq 2$) such that
\begin{align*}
	\Xi_t^2 \leq C (1-p) \gamma^2 \left(\frac{\phi_t^2}{p^2} + \frac{\sigma^2}{p}\right) \,.
\end{align*}
By smoothness,
\begin{align*}
	\phi_t^2 & = \frac{1}{n}\sum_{i=1}^n \norm{\nabla f_i(\xx_i^{(t)})}^2                                                                                                                                                                                \\
	         & \leq  \frac{3}{n}\sum_{i=1}^n \norm{\nabla f_i(\xx_i^{(t)}) - \nabla f(\xx_i^{(t)}) }^2 + \frac{3}{n}\sum_{i=1}^n \norm{\nabla f(\xx_i^{(t)}) - \nabla f(\bar \xx^{(t)}) }^2 + \frac{3}{n}\sum_{i=1}^n \norm{\nabla f(\bar \xx^{(t)}) }^2 \\
	         & \leq 3 \zeta^2  + 3 L^2 \Xi_t^2 + 3 \norm{\nabla f(\bar \xx^{(t)}) }^2\,.
\end{align*}
Supposing $(1-p)\gamma^2 \leq \frac{p^2}{6 C L^2}$, we can therefore estimate
\begin{align*}
	\Xi_t^2 & \leq C (1-p) \gamma^2 \left(\frac{3 \norm{\nabla f(\bar \xx^{(t)}) }^2 + 3L^2 \Xi_t^2 + 3 \zeta^2}{p^2} + \frac{\sigma^2}{p}\right)      \\
	        & \leq 3C (1-p) \gamma^2 \left(\frac{ \norm{\nabla f(\bar \xx^{(t)}) }^2 + \zeta^2}{p^2} + \frac{\sigma^2}{p}\right) + \frac{1}{2} \Xi_t^2 \,,
\end{align*}
and hence
\begin{align}
	\Xi_t^2 & \leq 6C (1-p) \gamma^2 \left(\frac{ \norm{\nabla f(\bar \xx^{(t)}) }^2}{p^2} + \frac{\zeta^2}{p^2}+ \frac{\sigma^2}{p}\right) \,. \label{eq:last}
\end{align}
The claimed bounds can now easily be verified, by plugging the provided values into~\eqref{eq:last}. For simplicity in the main text we assume that $\zeta=0$ (we are in the datacenter training scenario).
\paragraph{Small $\gamma$.} By choosing $\gamma \leq \frac{p}{4 n L C}$, we check that our previous constraint $\gamma^2 \stackrel{C\geq 2}{\leq} \frac{p^2}{6 C L^2}$ is satisfied, and
\begin{align*}
	\eqref{eq:last} \leq  \frac{ \norm{\nabla f(\bar \xx^{(t)}) }^2}{ 4 n^2 C L^2} +  \frac{\gamma  \sigma^2}{nL} \stackrel{C\geq 2}{\leq} \eqref{eq:want} \,.
\end{align*}

\paragraph{Small $p$.} By choosing $1-p \leq \frac{1}{5C(1+\gamma L n)}$, we note that $p \stackrel{C\geq 2}{\geq} \frac{9}{10}$.
Moreover, our previous constraint $(1-p)\gamma^2 \leq \frac{\gamma^2}{5 C} \leq \frac{p^2}{6L^2 C}$ is satisfied (note that $\gamma \leq \frac{1}{4L}$ throughout).
Hence
\begin{align*}
	\eqref{eq:last} \leq \frac{4 \gamma^2}{5(1+\gamma Ln)} \left( \frac{  100 \norm{  \nabla f(\bar \xx^{(t)}) }^2 }{ 81 } + \frac{10 \sigma^2}{9}\right) \stackrel{\gamma \leq 1/(4L)}{\leq} %
	\eqref{eq:want}
\end{align*}

In the above calculations we for the simplicity assumed that $\zeta = 0$. For the general non-iid data case when $\zeta > 0$ we can calculate similar bounds on $\gamma$, $p$. These bounds would have similar dependence on parameters, and would be stricter. Indeed, the typical consensus distance would be also influenced by non-iidness of the data $\zeta$ and it is therefore harder to satisfy the CCD condition.

\subsection{Proof of Lemma~\ref{lem:rep_gossip}, repeated gossip}~\label{subsec:repeated_gossip}
By the assumption stated in the lemma, it holds for each component $\mW_i$ of the product $\mW=\mW_k\dots \mW_1$, $i \in [1, k]$ that
\begin{align*}
	\textstyle
	\E_{\mW_i} \norm{ \mX \mW_i - \bar{\mX} }_F^2  \leq (1 - p) \norm{\mX - \bar{\mX}}_F^2
	, \forall \mX \in \R^{d \times n}.
\end{align*}
Now lets estimate the parameter $p_{\mW}$. Using that $\mW_i$ are independent
\begin{align*}
	\textstyle
	\E_{\mW} \norm{ \mX \mW - \bar{\mX} }_F^2 & = \E_{\mW_1\dots \mW_k} \norm{ \mX \mW_k\dots \mW_1 - \bar{\mX} }_F^2 =  \\
	                                          & {=} \E_{\mW_2\dots \mW_k}  \E_{\mW_1}\norm{ \mY \mW_1 - \bar{\mY} }_F^2,
\end{align*}
where we defined $\mY = \mX \mW_k\dots\mW_2$ and used that $\mW_i \frac{1}{n}\1\1^\top = \frac{1}{n}\1\1^\top$, so
\begin{align*}
	\bar \mY = \mX \mW_k\dots\mW_2 \frac{1}{n}\1\1^\top = \mX \frac{1}{n}\1\1^\top = \bar \mX.
\end{align*}
Therefore,
\begin{align*}
	\textstyle
	\E_{\mW} \norm{ \mX \mW - \bar{\mX} }_F^2 & \leq (1 - p)\E_{\mW_2\dots \mW_k}  \norm{ \mX \mW_k\dots\mW_2 - \bar{\mX} }_F^2.
\end{align*}
Applying the same calculations to the rest, we conclude that $1 - p_\mW = (1 - p)^k$.

\clearpage

\section{Detailed Experimental Setup} \label{sec:detailed_exp_setup}
\paragraph{Comments on large-batch training.}
Coupling the quality loss issue of the decentralized training with the large-batch training difficulty
is non-trivial and is out of the scope of this paper.
Instead, we use reasonable local mini-batch sizes (together with the number of workers (denoted as $n$)),
as well as the well-developed large-batch training techniques~\citep{goyal2017accurate},
to avoid the difficulty of extreme large-batch training.

\paragraph{Multi-phase experiment justification.}
The averaged local gradient norm $\phi_{t}$ as well as the $L$-smoothness
of ResNet-20 on CIFAR-10 for a ring and a complete graph ($n\!=\!32$)
are shown in Figure~\ref{fig:experiment_setup_motivation_ring} and Figure~\ref{fig:experiment_setup_motivation_complete_graph} respectively.

The estimation procedure is analogous to that in~\cite{santurkar2018does,lin2020extrapolation}:
we take 8 additional steps long the direction of current update, each with $0.2$ of normal step size.
This is calculated at every 8 training steps.
The smoothness is evaluated as the maximum value of $L$ satisfying Assumption~\ref{a:lsmooth_nc}.

\begin{figure*}[!h]
	\centering
	\vspace{-1.em}
	\subfigure[The consensus distance for decentralized training.]{
		\includegraphics[width=.305\textwidth,]{figures/three_phase_motivation/consensus_distance_cifar_ring_32.pdf}~\label{subfig:uncontrolled_consensus_dist}
	}
	\hfill
	\subfigure[The averaged norm of the local gradients for decentralized training.]{
		\includegraphics[width=.305\textwidth,]{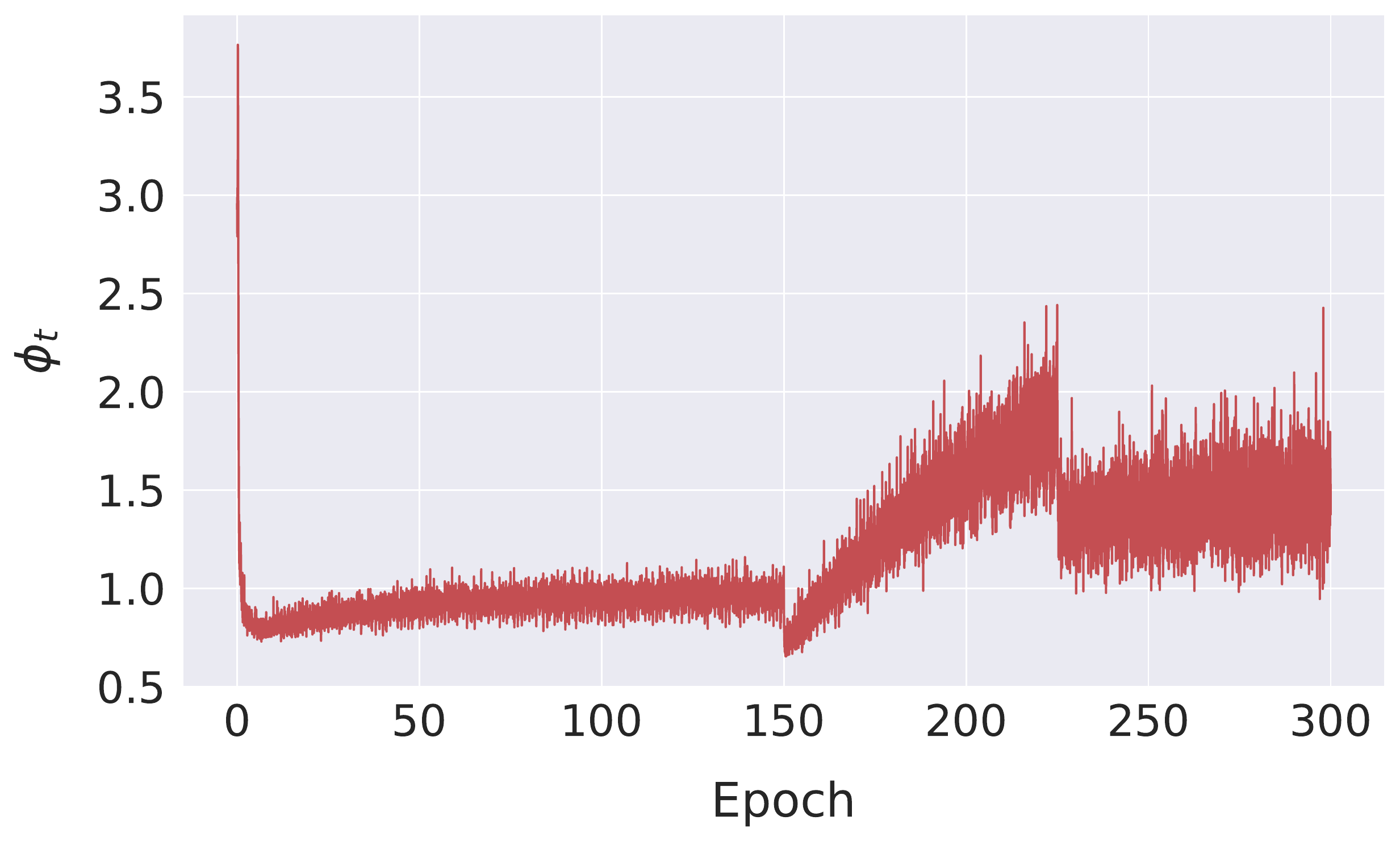}~\label{subfig:local_grad_norm}
	}
	\hfill
	\subfigure[The gradient Lipschitz curve for decentralized training.]{
		\includegraphics[width=.305\textwidth,]{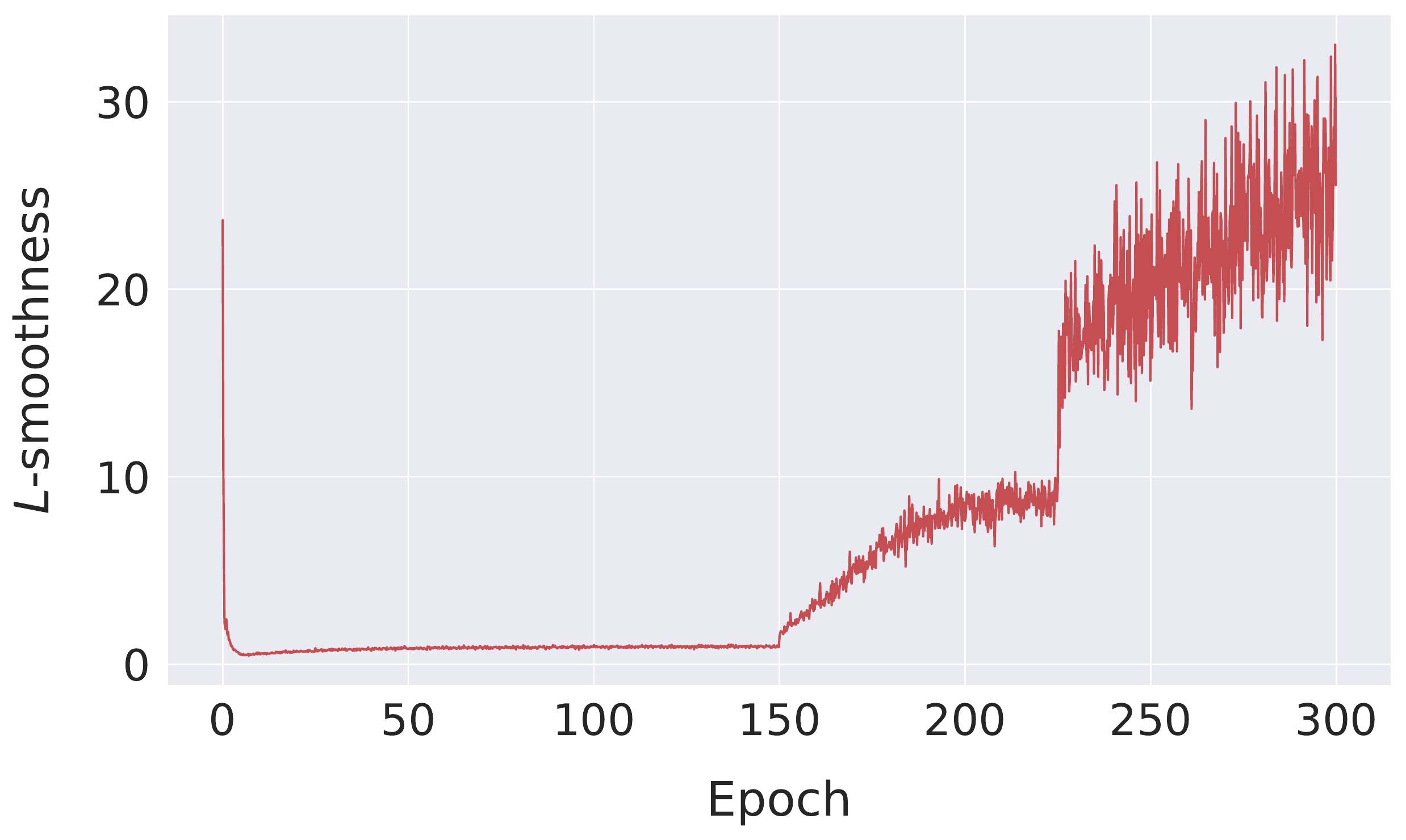}~\label{subfig:grad_lips}
	}
	\vspace{-1em}
	\caption{\small
		\textbf{Justification for our multiple-phase experimental design choice} (on ring graph).
		We run ResNet-20 on CIFAR-10 ($n\!=\!32$) with the ring topology.
		We can observe the three quantities most relevant to optimization all naturally form three phases, dictated by the learning rate schedule.
	}
	\vspace{-0.5em}
	\label{fig:experiment_setup_motivation_ring}
\end{figure*}

\begin{figure*}[!h]
	\centering
	\subfigure[The averaged norm of the local gradients for centralized training.]{
		\includegraphics[width=.405\textwidth,]{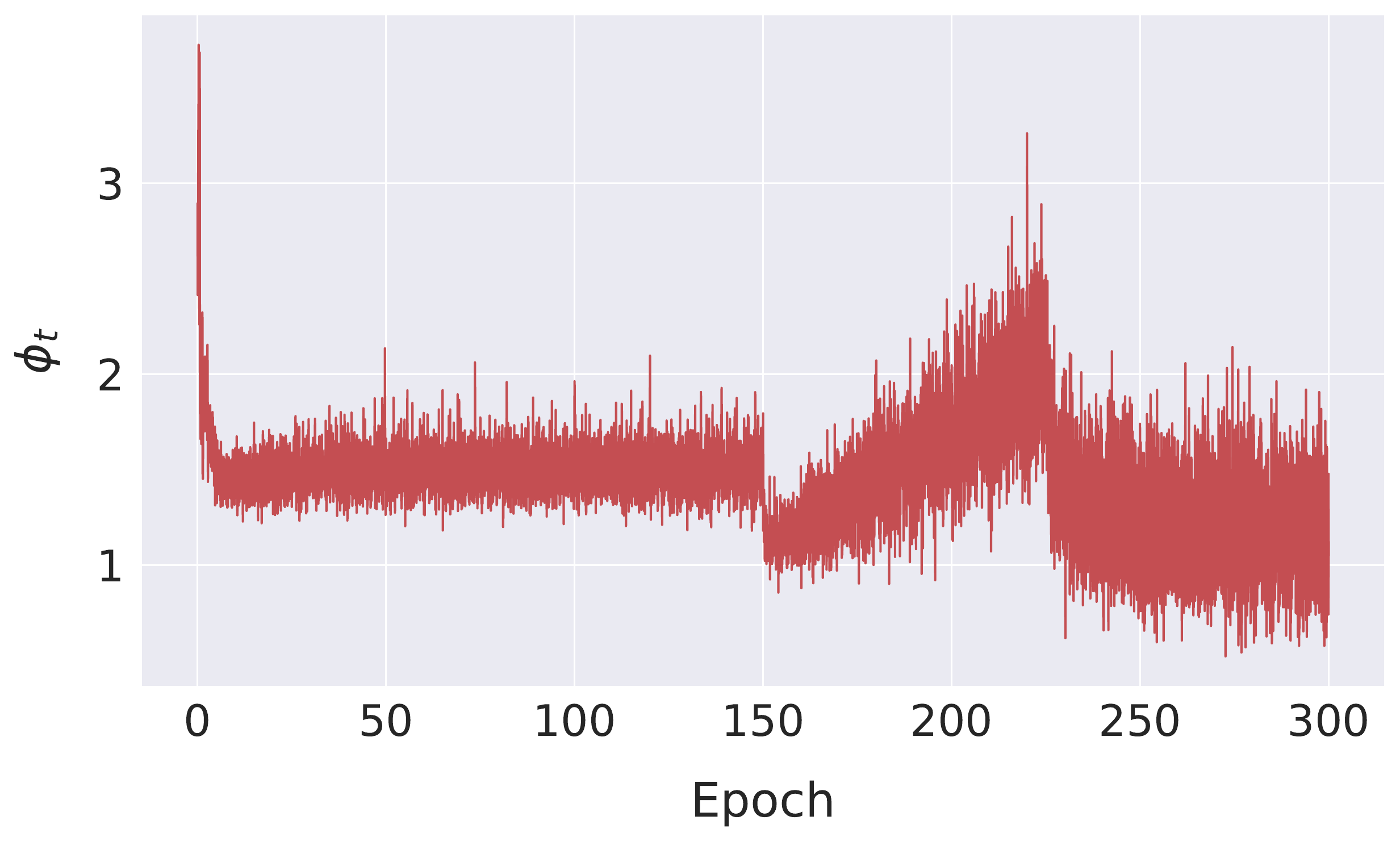}
	}
	\hfill
	\subfigure[The gradient Lipschitz curve for centralized training.]{
		\includegraphics[width=.405\textwidth,]{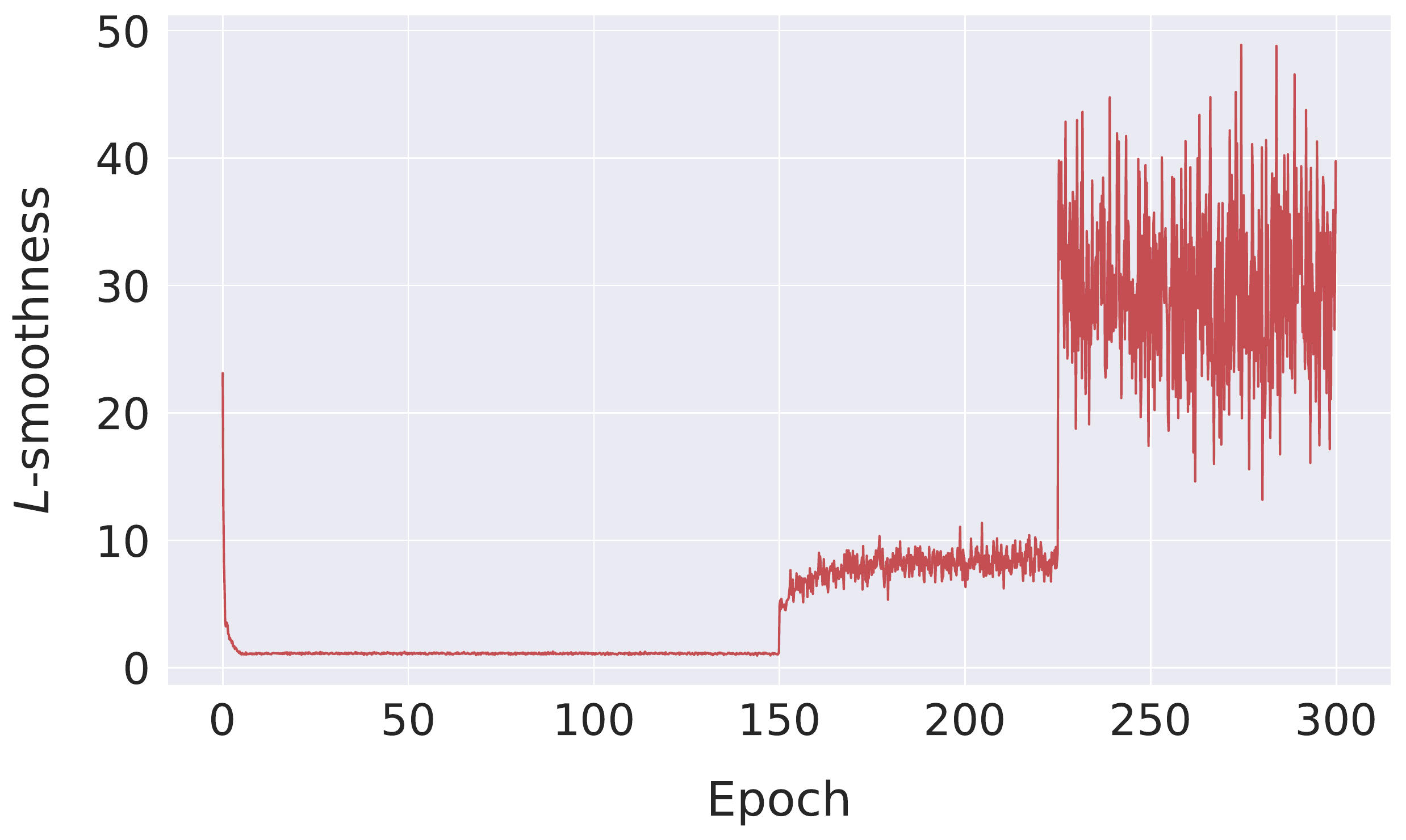}
	}
	\vspace{-1em}
	\caption{\small
		\textbf{Justification for our multiple-phase experimental design choice} (on complete graph).
		We run ResNet-20 on CIFAR-10 ($n\!=\!32$) with the complete topology.
		We can again observe the three quantities most relevant to optimization all naturally form three phases, dictated by the learning rate schedule.
	}
	\vspace{-0.5em}
	\label{fig:experiment_setup_motivation_complete_graph}
\end{figure*}

\section{Additional Results}
\subsection{Understanding on Consensus Averaging Problem} \label{subsec:consensus_averaging_problem}
We study a host of communication topologies: (1) deterministic topologies (ring, and complete graph)
and (2) undirected time-varying topologies (illustrated below).
\begin{itemize}[nosep,leftmargin=12pt]
	\item \textbf{Random matching}~\citep{boyd2006randomized}.
	      At each communication step, all nodes are divided into non-overlapping pairs randomly.
	      Each node connects all other nodes with equal probability.

	\item \textbf{Exponential graph}~\citep{assran2019stochastic}.
	      Each is assigned a rank from $0$ to $n-1$.
	      Each node $i$ periodically communicates with a list nodes with rank $i + 2^0, i + 2^1, \ldots, i + 2^{\lfloor \log_{2} (n-1) \rfloor}$.
	      In the one-peer-per-node experiments, each node only communicates to one node by cycling through its list.
	      The formed graph is undirected,
	      i.e., both transmission and reception take place in each communication.

	\item \textbf{Bipartite exponential graph}~\citep{lian2018asynchronous,assran2019stochastic}.
	      In order to avert deadlocks~\citep{lian2018asynchronous},
	      the node with an odd rank $i$ cycles through nodes with even ranks $i + 2^0 - 1, i + 2^1 - 1, \ldots, i + 2^{\lfloor \log_{2} (n-1) \rfloor} - 1$
	      by transmitting a message and waiting for a response.
	      while the nodes with even ranks only await messages and reply upon reception.
\end{itemize}

Table~\ref{tab:spectral_gap_node_degree} displays the spectral gap and node degree of studied topologies,
and Figure~\ref{fig:understanding_on_consensus_averaging_extended}
provides the convergence curves for different communication topologies on graph scales.
Figure~\ref{fig:understanding_on_consensus_averaging_spectral_gap} in addition
visualizes the spectral gap (in expectation) for different communication topologies.

\begin{table}[!h]
	\centering
	\caption{\small
		Spectral gap and node degree of studied topologies.
	}
	\vspace{-1em}
	\label{tab:spectral_gap_node_degree}
		\begin{tabular}{c|cc}
			\toprule
			Topologies                  & Spectral Gaps (in expectation) & Node degrees ($n$ nodes) \\ \midrule
			Complete                    & $1$                            & $n$                      \\
			Fixed ring                  & $\cO (\frac{1}{n^2})$          & $2$                      \\
			Exponential graph           & $\cO (1)$                      & $2$                      \\
			Bipartite exponential graph & $\cO (1)$                      & $1$                      \\
			Random matching             & $\cO (1)$                      & $1$                      \\
			\bottomrule
		\end{tabular}%
\end{table}

\begin{figure*}[!h]
	\centering
	\subfigure[$n=16$]{
		\includegraphics[width=.315\textwidth,]{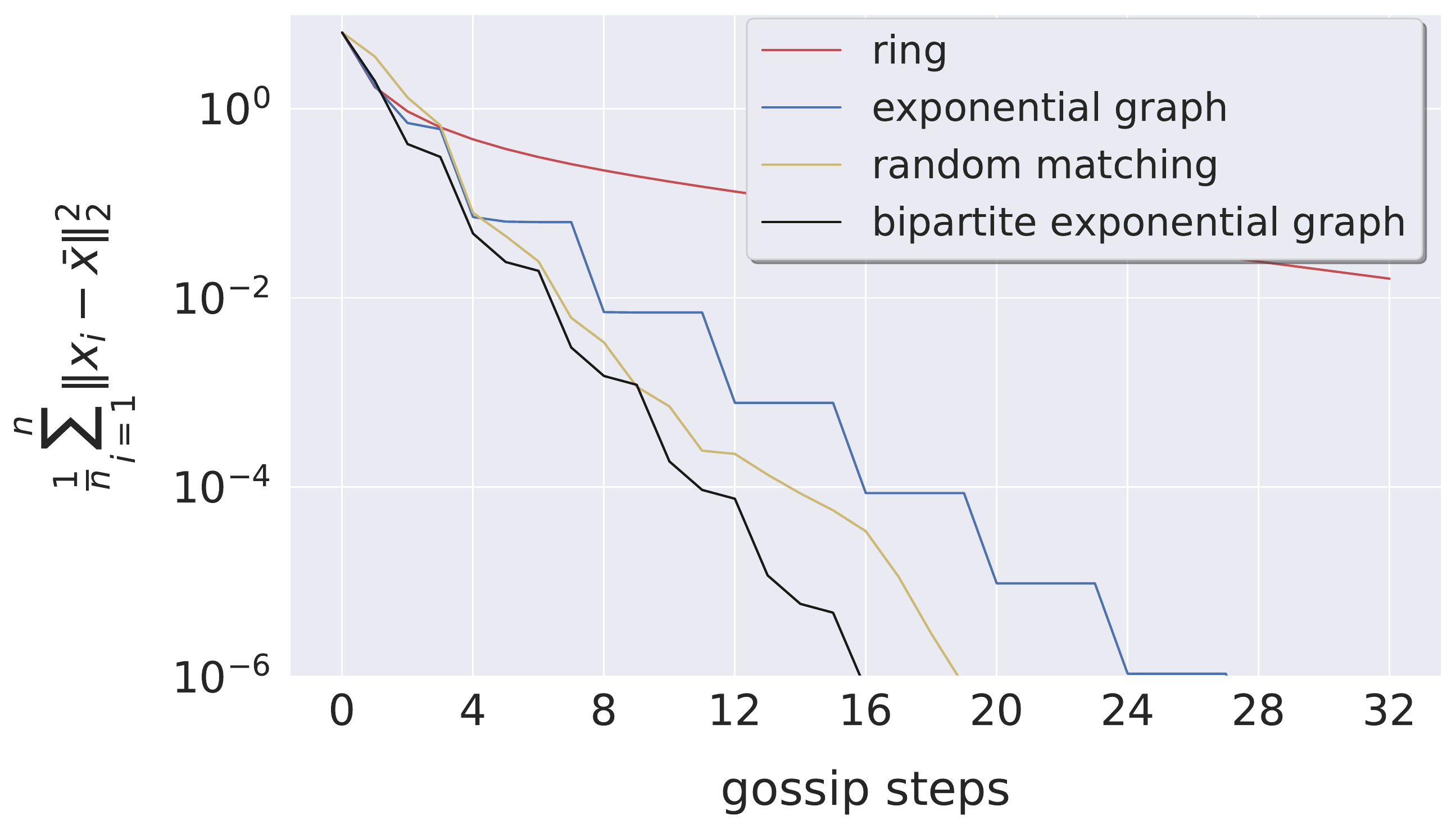}
	}
	\hfill
	\subfigure[$n=32$]{
		\includegraphics[width=.315\textwidth,]{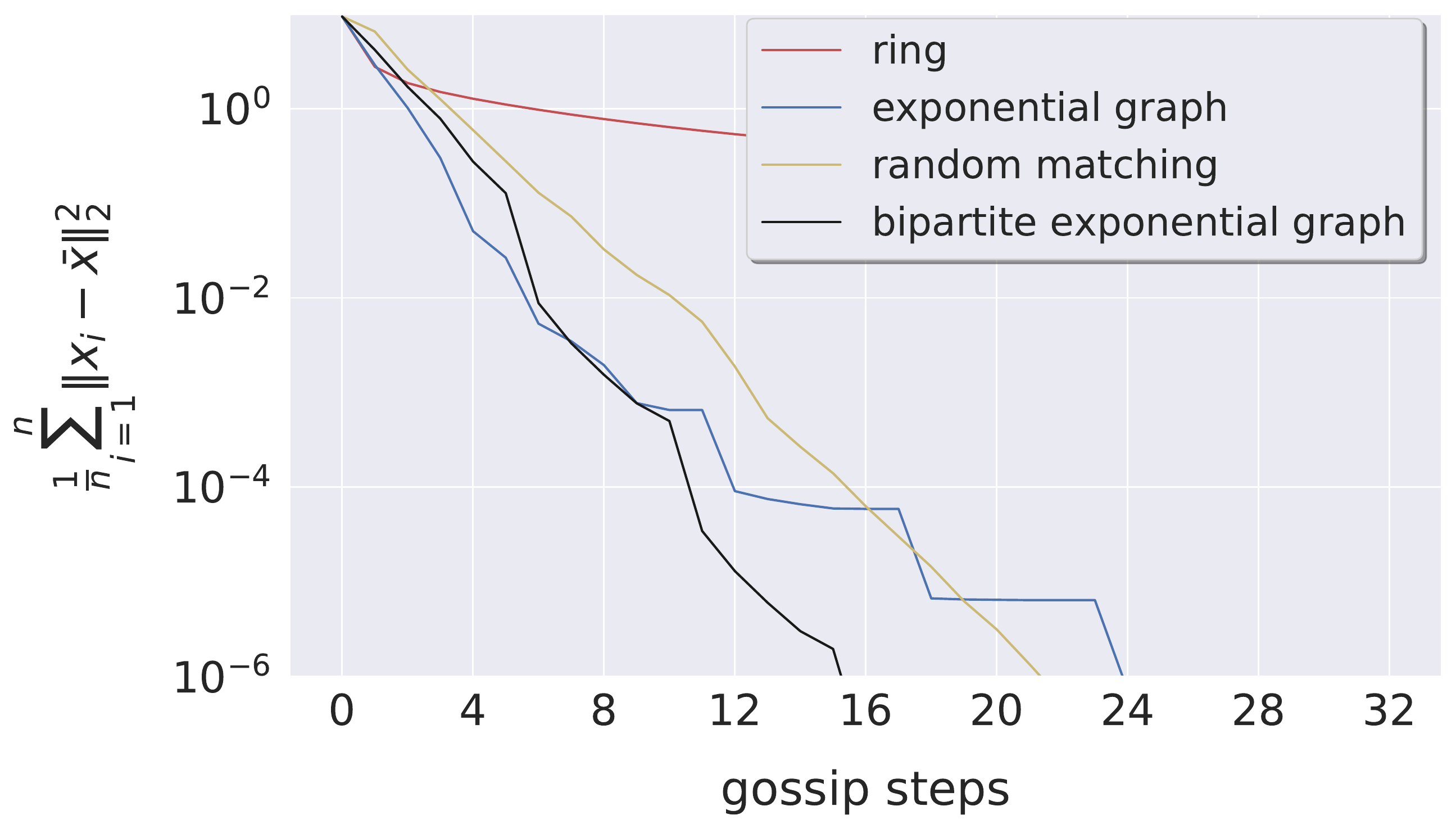}
	}
	\hfill
	\subfigure[$n=64$]{
		\includegraphics[width=.315\textwidth,]{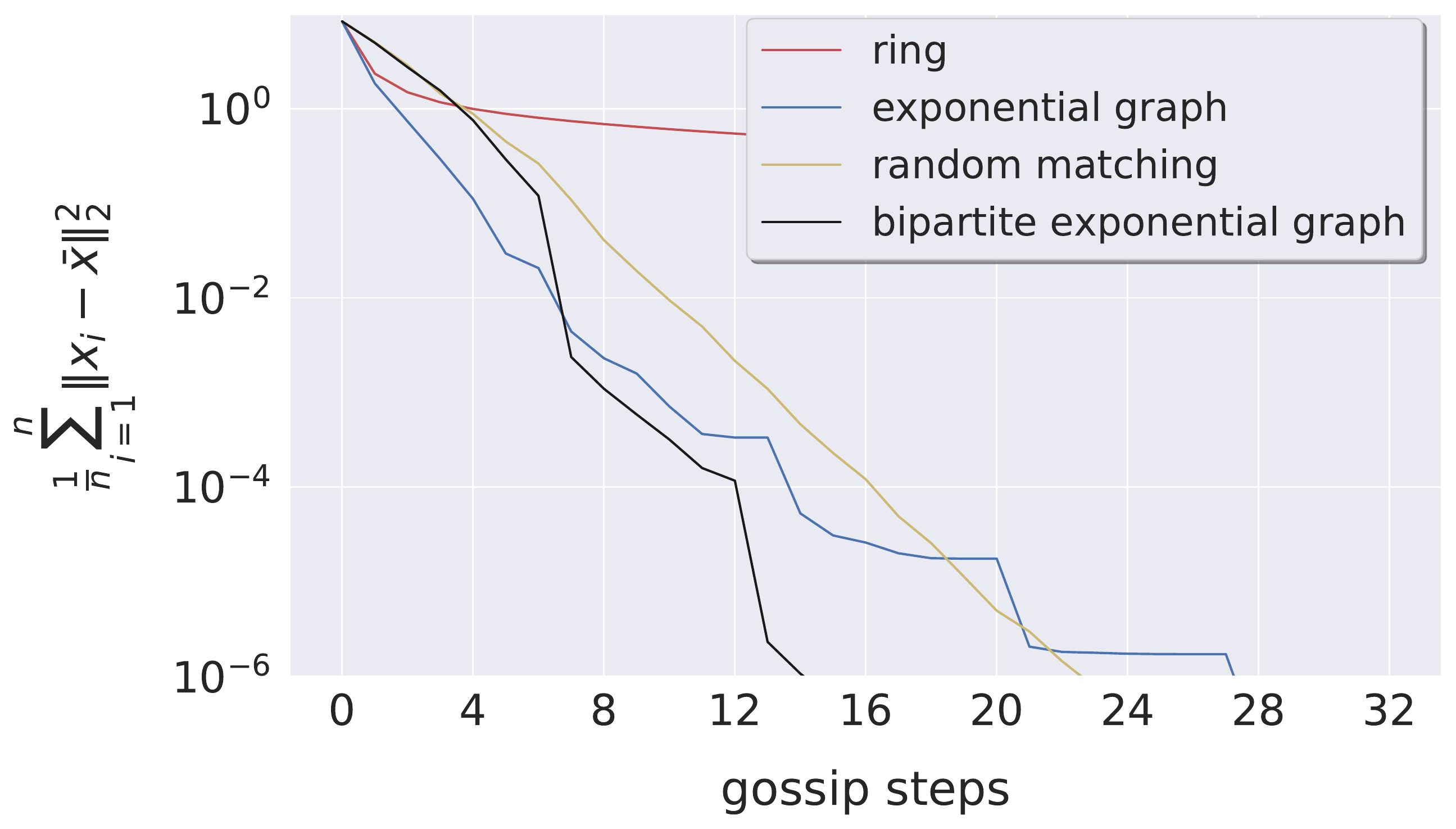}
	}
	\vspace{-1em}
	\caption{\small
		The convergence curves for the consensus averaging problem
		on different communication topologies and different scales (i.e., $n\!=\!16$, $n\!=\!64$ and $n\!=\!128$).
		This figure complements the Figure~\ref{fig:understanding_the_importance_of_communication_topologies} in the main text.
	}
	\label{fig:understanding_on_consensus_averaging_extended}
\end{figure*}

\begin{figure*}[!h]
	\centering
	\subfigure[$n=16$]{
		\includegraphics[width=.475\textwidth,]{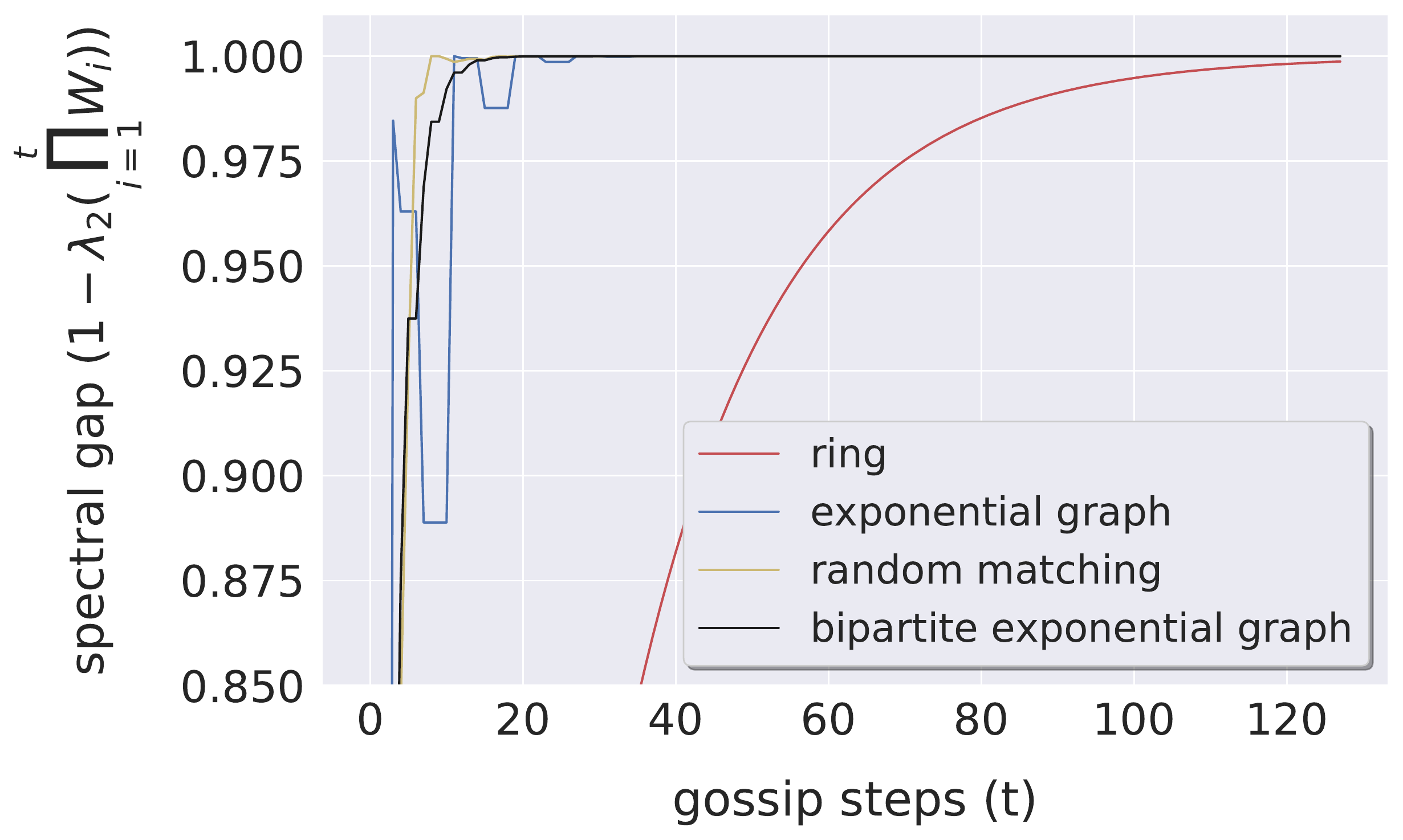}
	}
	\hfill
	\subfigure[$n=32$]{
		\includegraphics[width=.475\textwidth,]{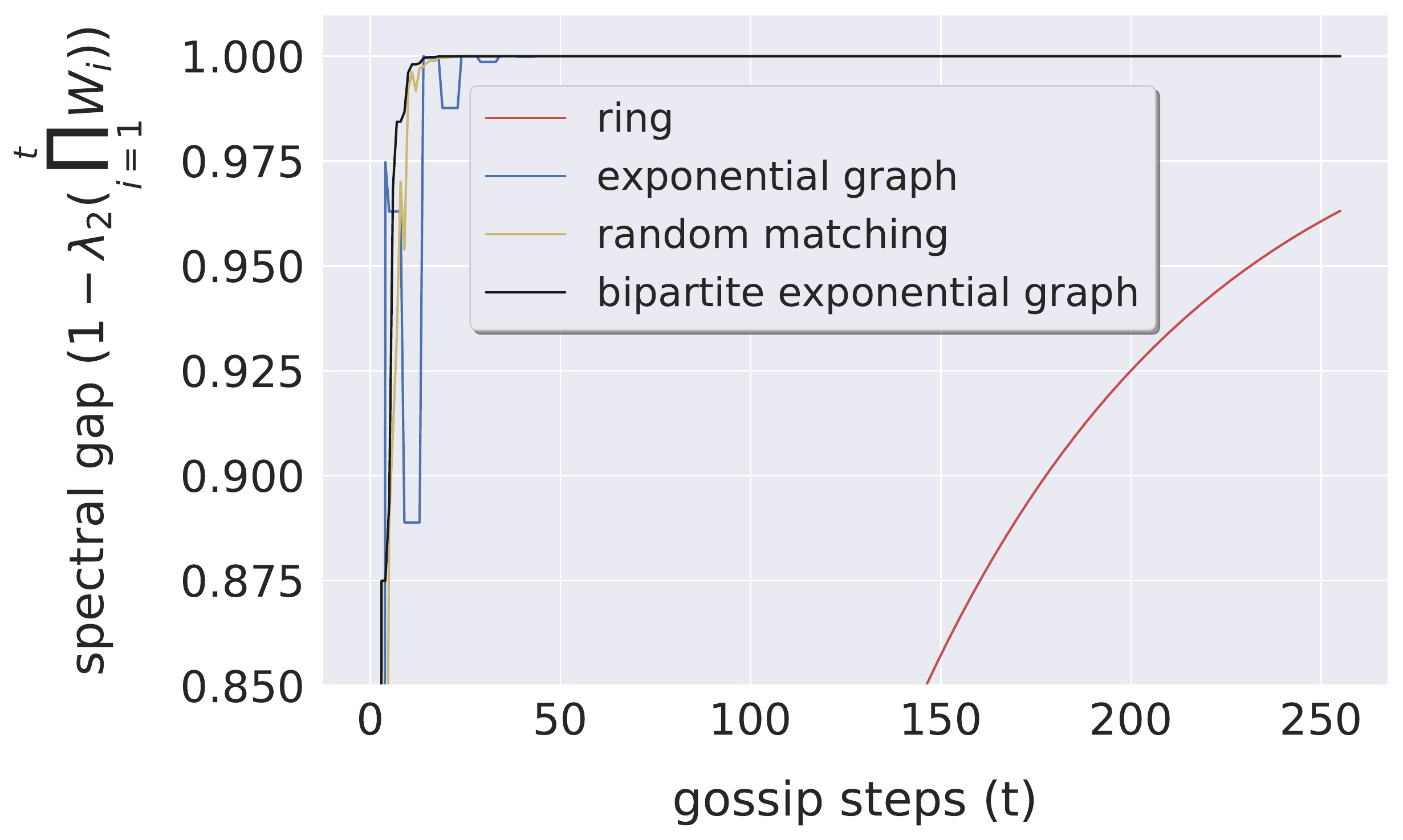}
	}
	\vspace{-1em}
	\caption{\small
		The spectral gap (in expectation) of different communication topologies on different graph scales.
	}
	\label{fig:understanding_on_consensus_averaging_spectral_gap}
\end{figure*}

Table~\ref{tab:understanding_the_impact_of_communication_topologies_on_dl_complete_results}
examines these topologies on a standard deep learning benchmark with different graph scales,
while Figure~\ref{fig:resnet20_cifar10_consensus_distance_vs_comm_rounds}
visualizes the required communication rounds (per gradient update step) for a range of consensus distance targets.

\begin{table*}[!h]
	\centering
		\begin{tabular}{cccccc}
			\toprule
			     & Complete         & Fixed ring       & Exponential graph & Bipartite exponential graph & Random matching  \\ \midrule
			n=16 & $92.91 \pm 0.12$ & $92.51 \pm 0.19$ & $92.63 \pm 0.30$  & $92.76 \pm 0.04$            & $92.65 \pm 0.15$ \\
			n=32 & $92.82 \pm 0.27$ & $91.93 \pm 0.05$ & $92.64 \pm 0.04$  & $92.29 \pm 0.15$            & $92.27 \pm 0.17$ \\
			\bottomrule
		\end{tabular}%
	\vspace{-0.5em}
	\caption{\small
		\textbf{The effect of communication topologies and scales}
		(ResNet-20 on CIFAR-10 with $n\!=\!32$).
		The test top-1 accuracies are over three seeds with fine-tuned learning rates.
	}
	\label{tab:understanding_the_impact_of_communication_topologies_on_dl_complete_results}
\end{table*}

\begin{figure*}[!h]
	\centering
	\includegraphics[width=.6\textwidth,]{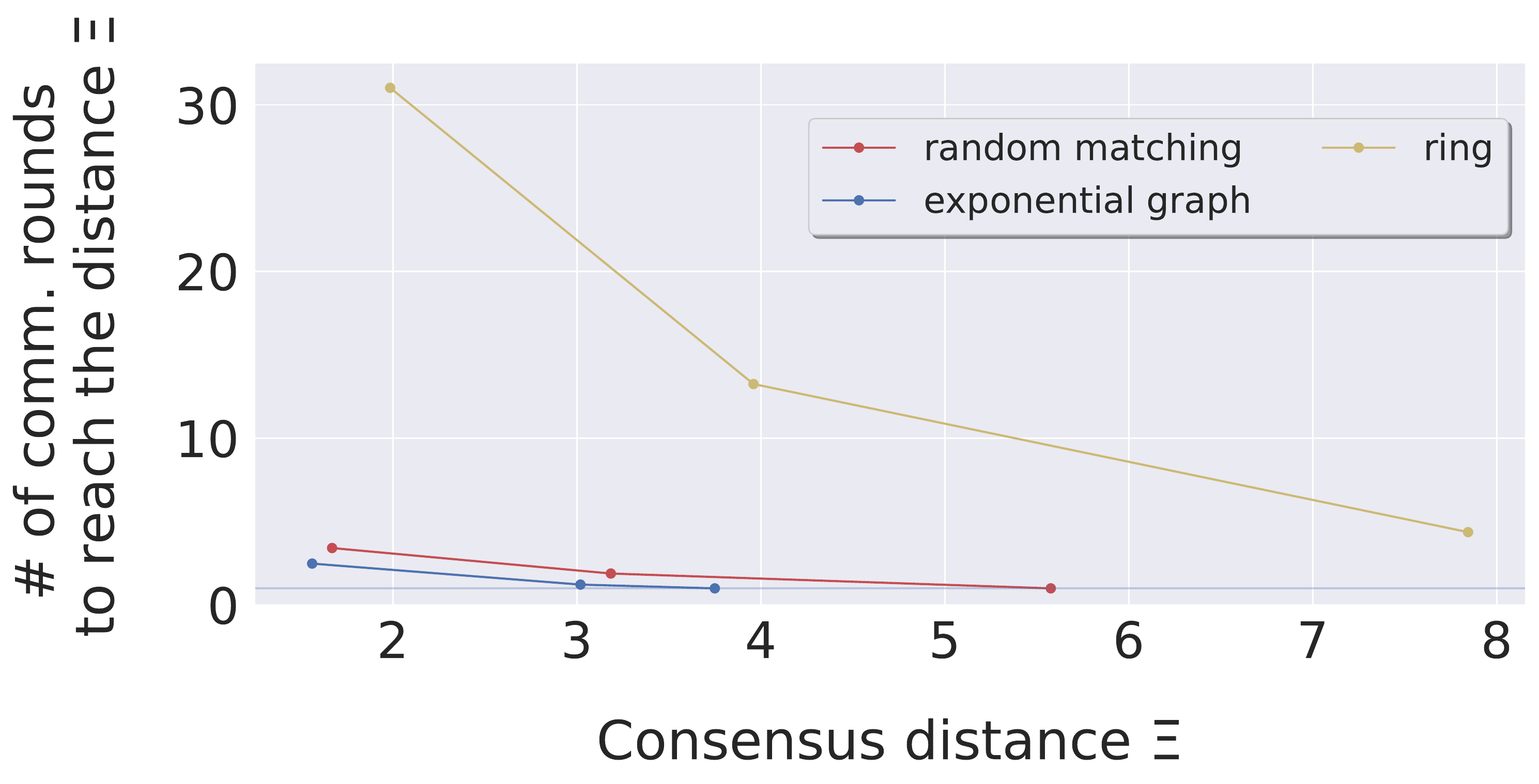}
	\vspace{-1em}
	\caption{\small
		\textbf{Target consensus distance v.s.\ the required communication rounds} (per gradient update step),
		for training ResNet-20 on CIFAR-10 with different communication topologies.
		We focus on the setup of dec-phase-1
		and vary the target consensus distance for different communication topologies.
		Due to the changing consensus distance over the training (of the interested phase-1),
		we consider the averaged consensus distance.
		The topologies of exponential graph and random matching,
		empower the capability of fast convergence in gossip averaging
		and thus only a few steps are required to reach the target consensus distance.
	}
	\label{fig:resnet20_cifar10_consensus_distance_vs_comm_rounds}
\end{figure*}

\subsection{Understanding the Decentralized Deep Learning Training for CV Tasks} \label{sec:more_understanding}
We use ring as our underlying decentralized communication topology in this subsection.

\paragraph{Elaborated results on consensus distance control.}
Table~\ref{tab:elaborated_resnet20_cifar10_different_consensus_distances_and_phases_by_constant_on_ring}
is the elaborated version of Table~\ref{tab:resnet20_cifar10_different_consensus_distances_and_phases_by_constant_on_ring}
with more evaluated consensus distances.

\begin{table}[!h]
	\centering
	\caption{\small
		\textbf{The impact of consensus distance of different phases on generalization performance} (test top-1 accuracy)
		of training ResNet-20 on CIFAR-10.
		The centralized baseline performance for $n\!=\!32$ and $n\!=\!64$
		are $92.82 \pm 0.27$ and $92.71 \pm 0.11$ respectively.
		The performance of decentralized training (all phases on a fixed ring and w/o consensus distance control)
		for $n\!=\!32$ and $n\!=\!64$ are $91.74 \pm 0.15$ and $89.87 \pm 0.12$ respectively.
	}
	\vspace{-1em}
	\label{tab:elaborated_resnet20_cifar10_different_consensus_distances_and_phases_by_constant_on_ring}
	\resizebox{1.\textwidth}{!}{%
		\huge
		\begin{tabular}{c|ccccc|cccc|ccc|ccc}
			\toprule
			& \multicolumn{5}{|c|}{dec-phase-1}            & \multicolumn{4}{c|}{dec-phase-2}                          & \multicolumn{3}{c|}{dec-phase-3}             & \multicolumn{3}{c}{dec-phase-2 + dec-phase-3}         \\ \midrule
			     & $\Ximax$         & 1/2 $\Ximax$     & 1/4 $\Ximax$     & 1/8 $\Ximax$     & 1/16 $\Ximax$    & $\Ximax$         & 1/2 $\Ximax$     & 1/4 $\Ximax$      & 1/40 $\Ximax$    & $\Ximax$         & 1/2 $\Ximax$     & 1/4 $\Ximax$     & $\Ximax$         & 1/2 $\Ximax$     & 1/4 $\Ximax$     \\ \midrule
			n=32 & $91.78 \pm 0.35$ & $92.36 \pm 0.21$ & $92.74 \pm 0.10$ & $92.77 \pm 0.25$ & $92.72 \pm 0.05$ & $93.04 \pm 0.01$ & $92.99 \pm 0.30$ & $92.87 \pm 0.11$  & $92.84 \pm 0.27$ & $92.60 \pm 0.00$ & $92.82 \pm 0.21$ & $92.85 \pm 0.24$ & $92.94 \pm 0.07$ & $93.03 \pm 0.24$ & $92.93 \pm 0.15$ \\
			n=64 & $90.31 \pm 0.12$ & $92.18 \pm 0.07$ & $92.45 \pm 0.17$ & -                & -                & $93.14 \pm 0.04$ & $92.94 \pm 0.10$ & $92.79 \pm 0.07 $ & -                & $92.23 \pm 0.12$ & $92.50 \pm 0.09$ & $92.60 \pm 0.10$ & $92.95 \pm 0.07$ & $92.83 \pm 0.12$ & $92.66 \pm 0.07$ \\
			\bottomrule
		\end{tabular}
	}
\end{table}

\paragraph{SlowMo cannot fully address the decentralized optimization/generalization difficulty.}
Table~\ref{tab:resnet20_cifar10_the_impact_of_slowmo} studies the effectiveness of using SlowMo for better decentralized training.
We can witness that even though the performance of decentralized training can be boosted to some extent,
it cannot fully address the quality loss issue brought by decentralized training.

\begin{table}[!h]
	\centering
	\caption{\small
		\textbf{The effect of SlowMo for decentralized learning}, for training ResNet20 on CIFAR-10 ($n=32$).
		The results (over three random seeds) use the tuned hyper-parameter of SlowMo
		mentioned in the original paper~\citep{Wang2020SlowMo}.
		The centralized baseline performance is $92.82 \pm 0.27$.
	}
	\vspace{-1em}
	\label{tab:resnet20_cifar10_the_impact_of_slowmo}
	\begin{tabular}{ccc}
		\toprule
		topology          & w/o SlowMo       & w/ SlowMo        \\ \midrule
		exponential graph & $92.63 \pm 0.22$ & $92.42 \pm 0.36$ \\
		ring              & $91.74 \pm 0.15$ & $92.53 \pm 0.10$ \\
		\bottomrule
	\end{tabular}
\end{table}

\paragraph{On the ineffectiveness of tuning learning rate.}
Table~\ref{tab:resnet20_cifar10_phase1_finetuned_learning_rates_on_ring}
displays the results of training ResNet-20 on CIFAR-10 (32 nodes),
with fine-tuned learning rate on phase-1;
learning rate tuning cannot address the test quality loss issue caused by the large consensus distance (i.e.\ over the CCD).

\begin{table}[!h]
	\centering
	\caption{\small
		\textbf{Phase-1 consensus distance control performance with fine-tuned learning rates} of training ResNet-20 on CIFAR-10 ($n\!=\!32$).
		Setup in this table is identical to that of Table~\ref{tab:resnet20_cifar10_different_consensus_distances_and_phases_by_constant_on_ring},
		except that we fine-tune the learning rate for each case from a grid of linear scaling-up factors $\{ 30, 28, 26, 24, 22 \}$.
		The results are over three seeds.
	}
	\vspace{-1em}
	\label{tab:resnet20_cifar10_phase1_finetuned_learning_rates_on_ring}
	\resizebox{.7\textwidth}{!}{%
		\begin{tabular}{cccc}
			\toprule
			                                 & $\Ximax$         & 1/2 $\Ximax$     & 1/4 $\Ximax$     \\ \midrule
			w/ tuned lr from the search grid & $91.95 \pm 0.26$ & $92.35 \pm 0.24$ & $92.54 \pm 0.08$ \\
			w/ default lr                    & $91.78 \pm 0.35$ & $92.36 \pm 0.21$ & $92.74 \pm 0.10$ \\
			\bottomrule
		\end{tabular}%
	}
\end{table}

\paragraph{Prolonged training for dec-phase-2 and dec-phase-3.}
Table~\ref{tab:resnet20_cifar10_ring_prolonged_phase_2_3} shows the results for prolonged dec-phase-2 and dec-phase-3
on CIFAR-10 with ResNet20. We can observe although longer training duration increases the performance, the improvement is rather small.

\begin{table}[!h]
	\centering
	\caption{\small
		\textbf{The impact of different numbers of training epochs (at phase-2 and phase-3)} on generalization,
		for training ResNet-20 on CIFAR-10 (ring topology with $n\!=\!32$).
		The number of epochs at phase-1 is chosen from $\{ 75, 100, 125 \}$,
		while the rest of the training reuses our default setup.
		Experiments are run over 2 seeds.
	}
	\vspace{-1em}
	\label{tab:resnet20_cifar10_ring_prolonged_phase_2_3}
		\begin{tabular}{cccc|ccc}
			\toprule
			\multirow{2}{*}{\diagbox{\# nodes}{target $\Xi$}}  & \multicolumn{3}{c|}{dec-phase-2}            & \multicolumn{3}{c}{dec-phase-3}  \\  \cmidrule(lr){2-7}
			             & $\Ximax$         & 1/2 $\Ximax$     & 1/4 $\Ximax$     & $\Ximax$          & 1/2 $\Ximax$     & 1/4 $\Ximax$     \\ \midrule
			$75$ epochs  & $93.04 \pm 0.01$ & $92.99 \pm 0.30$ & $92.87 \pm 0.11$ & $92.60 \pm 0.00 $ & $92.82 \pm 0.21$ & $92.85 \pm 0.24$ \\
			$100$ epochs & $93.08 \pm 0.08$ & $93.05 \pm 0.16$ & $92.94 \pm 0.03$ & $92.86 \pm 0.16$  & $92.90 \pm 0.18$ & $92.93 \pm 0.19$ \\
			$125$ epochs & $93.19 \pm 0.16$ & $93.11 \pm 0.17$ & $93.06 \pm 0.07$ & $92.87 \pm 0.23$  & $92.99 \pm 0.25$ & $92.97 \pm 0.20$ \\
			\bottomrule
		\end{tabular}%
\end{table}

\paragraph{The impact of half cosine learning rate schedule.}
Table~\ref{tab:resnet20_cifar10_ring_cosine_lr_annealing}
examines the existence of the critical consensus distance
with half cosine learning schedule
(this scheme is visited in~\citep{he2019bag} as a new paradigm for CNN training).
We can witness from Table~\ref{tab:resnet20_cifar10_ring_cosine_lr_annealing}
that the effect of critical consensus distance can be generalized to this learning rate schedule:
there exists a critical consensus distance in the initial training phase
(as revealed in the inline Figure of Table~\ref{tab:resnet20_cifar10_ring_cosine_lr_annealing})
and ensures good optimization and generalization.

\begin{table}[!h]
	\centering
	\caption{\small
		\textbf{The impact of half cosine learning rate schedule} on generalization,
		for training ResNet20 on CIFAR-10 (ring topology with $n\!=\!32$).
		The inline figure depicts the uncontrolled consensus distance over the whole training procedure
		through the half-cosine learning rate schedule.
		Only one training phase is considered for the consensus distance control
		and the numerical results in the table are averaged over 3 seeds.
	}
	\vspace{-1em}
	\label{tab:resnet20_cifar10_ring_cosine_lr_annealing}
	\includegraphics[width=.5\textwidth,]{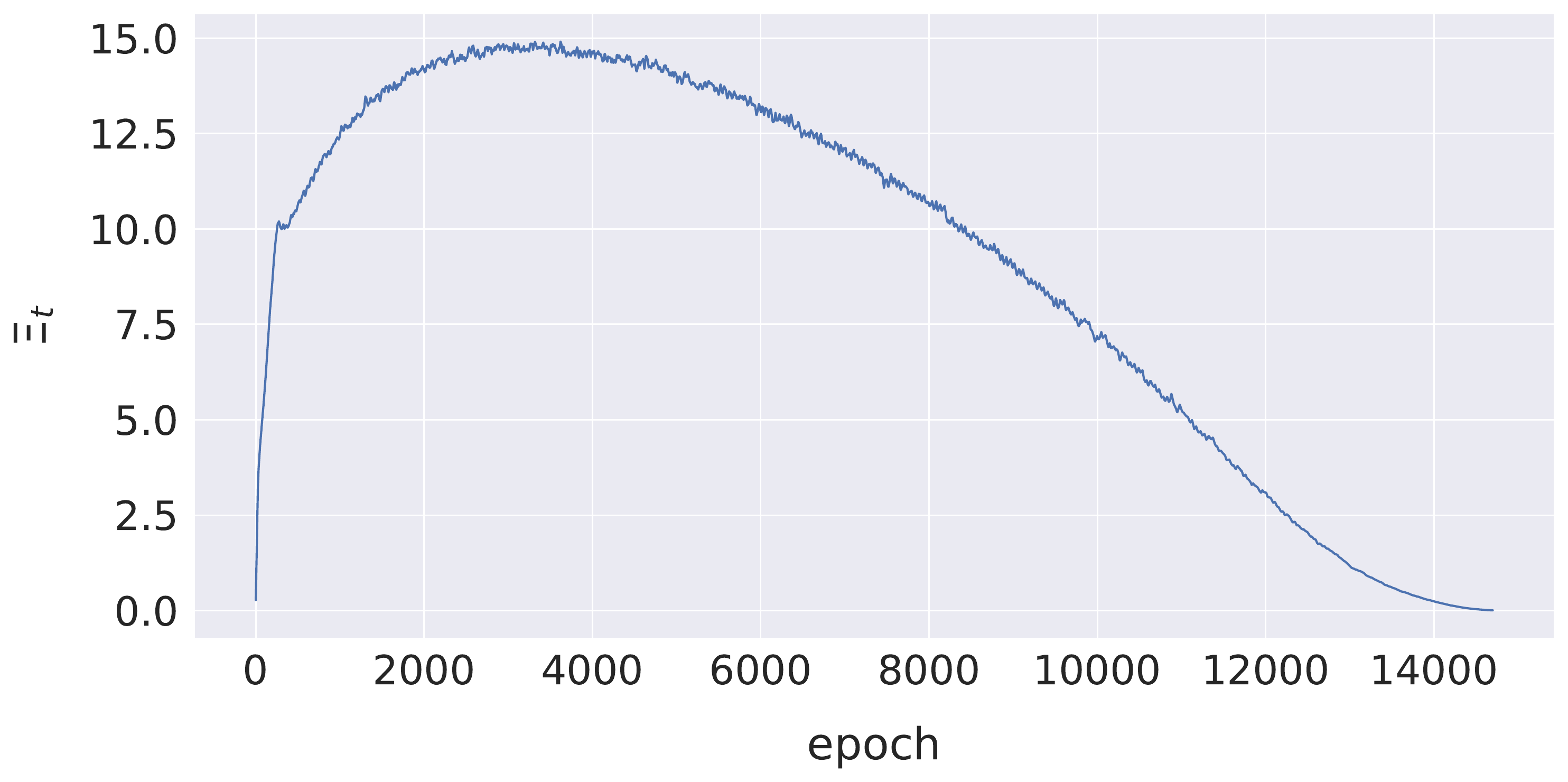}
	\begin{tabular}{ccccc}
		\toprule
		Ring ($\Ximax$)  & Ring ($1/2 \Ximax$) & Ring ($1/4 \Ximax$) & Ring ($1/8 \Ximax$) & Complete         \\ \midrule
		$92.10 \pm 0.06$ & $92.40 \pm 0.10$    & $92.83 \pm 0.11$    & $92.78 \pm 0.05$    & $92.84 \pm 0.22$ \\
		\bottomrule
	\end{tabular}
\end{table}

\clearpage
\subsubsection{Adaptive consensus distance control}~\label{subsec:adaptive_consensus_distance}
In Table~\ref{tab:resnet20_cifar10_different_consensus_distances_and_phases_by_ratio_on_ring},
we apply the adaptive consensus distance control in the experiments.
The observations are consistent with those in constant consensus distance control experiments.

\begin{table}[!h]
	\centering
	\caption{\small
		\textbf{
			The impact of different consensus distances on optimization and/or generalization,
			for different phases} of training ResNet-20 on CIFAR-10 ($n\!=\!32$).
		The table is almost identical to Table~\ref{tab:resnet20_cifar10_different_consensus_distances_and_phases_by_constant_on_ring},
		except the consensus distance is controlled by the (runtime) averaged norm of the local gradients (i.e. adaptive consensus distance).
	}
	\vspace{-1em}
	\label{tab:resnet20_cifar10_different_consensus_distances_and_phases_by_ratio_on_ring}
	\resizebox{.8\textwidth}{!}{%
		\begin{tabular}{c|ccccc}
			\toprule
			        & $\Ximax$         & $4 \phiema_t$    & $2 \phiema_t$    & $\phiema_t$      & $0.5 \phiema_t$  \\ \midrule
			Phase 1 & $91.78 \pm 0.35$ & $91.65 \pm 0.31$ & $92.47 \pm 0.18$ & $92.63 \pm 0.04$ & $92.80 \pm 0.16$ \\
			Phase 2 & $93.04 \pm 0.01$ & $93.05 \pm 0.18$ & $93.01 \pm 0.03$ & $93.03 \pm 0.08$ & $92.95 \pm 0.10$ \\
			Phase 3 & $92.94 \pm 0.07$ & $92.87 \pm 0.18$ & $92.83 \pm 0.20$ & -                & -                \\
			\bottomrule
		\end{tabular}%
	}
\end{table}

\subsection{Consensus control with other topologies} \label{subsec:other_topologies}
In Table~\ref{tab:resnet20_cifar10_phase_interference_on_undirected_exponential_graph},
we exert consensus control with an exponential graph as the base communication topology;
the local update step corresponds to the number of local model update steps per communication round,
and we use it as a way to increase discrepancy (consensus distance) among nodes.
We can observe that our findings from main experiments with a ring base topology are valid.
\begin{table}[!h]
	\centering
	\caption{\small
		The \textbf{impact of quality propagation across phases} (in both phase 1 and phase 2) on an \textbf{undirected time-varying exponential graph} ($n\!=\!32$),
		similar to Table~\ref{tab:resnet20_cifar10_phase_interference_on_ring}.
	}
	\label{tab:resnet20_cifar10_phase_interference_on_undirected_exponential_graph}
	\vspace{-1em}
	\resizebox{1.\textwidth}{!}{%
		\huge
		\begin{tabular}{lcccccccccccc}
			\toprule
			\multirow{2}{*}{\diagbox{phase 1}{phase 2}}  & \multicolumn{4}{c}{local update step = 1} & \multicolumn{4}{c}{local update step = 2}                      & \multicolumn{4}{c}{local update step = 4}                      \\ \cmidrule(lr){2-5} \cmidrule(lr){6-9} \cmidrule(lr){10-13}
			                 & $\Ximax$         & $2 \phiema_t$    & $\phiema_t$      & $0.5 \phiema_t$  & $\Ximax$         & $2 \phiema_t$    & $\phiema_t$      & $0.5 \phiema_t$  & $\Ximax$         & $2 \phiema_t$    & $\phiema_t$      & $0.5 \phiema_t$  \\ \midrule
			$2 \phiema_t$    & $92.43 \pm 0.16$ & $92.44 \pm 0.24$ & $92.36 \pm 0.06$ & $92.45 \pm 0.01$ & -                & -                & -                & -                & -                & -                & -                & -                \\
			$1 \phiema_t$    & $92.58 \pm 0.09$ & $92.37 \pm 0.14$ & $92.63 \pm 0.09$ & $92.51 \pm 0.16$ & -                & -                & -                & -                & -                & -                & -                & -                \\
			$0.5 \phiema_t$  & $92.74 \pm 0.17$ & $92.56 \pm 0.19$ & $92.56 \pm 0.21$ & $92.75 \pm 0.24$ & $92.79 \pm 0.13$ & $92.68 \pm 0.21$ & $92.65 \pm 0.07$ & $92.68 \pm 0.22$ & $92.85 \pm 0.09$ & $92.76 \pm 0.09$ & $92.72 \pm 0.21$ & $92.75 \pm 0.09$ \\
			$0.25 \phiema_t$ & $92.71 \pm 0.13$ & $92.72 \pm 0.08$ & $92.81 \pm 0.20$ & $92.76 \pm 0.24$ & $92.83 \pm 0.21$ & $92.86 \pm 0.16$ & $92.86 \pm 0.13$ & $92.81 \pm 0.26$ & $93.13 \pm 0.09$ & $92.88 \pm 0.16$ & $92.85 \pm 0.26$ & $92.77 \pm 0.23$ \\
			\bottomrule
		\end{tabular}%
	}
\end{table}

\subsubsection{The Existence of the Optimal Consensus Distance for Noise Injection.} \label{subsec:optimal_consensus_for_better_generalization}
Table~\ref{tab:resnet20_cifar10_phase2_on_undirected_time_varying_exponential_graph}
uses a different communication topology (i.e. time-varying exponential graph) for decentralized optimization.
Here exponential graph with large spectral gap is applied to CIFAR-10 dec-phase-2 training.
We apply the adaptive consensus distance control in this set of experiments.
We can observe that increasing consensus distance further by taking local steps improves generalization,
however, too many local steps diminish the performance. For instance, for ratio=2, the performance peaks at local update steps 2
and drops at local update 4.
It points out that an optimal consensus distance is required to inject proper stochastic noise for better generalization.

\begin{table}[!h]
	\centering
	\caption{\small
		\textbf{The impact of different consensus distances at phase 2},
		for training ResNet-20 on CIFAR-10 with time-varying exponential graph ($n\!=\!32$).
		The baseline performance of using exponential graph for the entire decentralized training is $92.64 \pm 0.04$.
		The reported test top-1 accuracies are averaged over three seeds.
	}
	\label{tab:resnet20_cifar10_phase2_on_undirected_time_varying_exponential_graph}
	\vspace{-1em}
	\resizebox{1.\textwidth}{!}{%
		\begin{tabular}{ccccccccccc}
			\toprule
			& \multicolumn{4}{c}{local update step = 1}             & \multicolumn{3}{c}{local update step = 2} & \multicolumn{3}{c}{local update step = 4} \\ \cmidrule(lr){2-5} \cmidrule(lr){6-8} \cmidrule(lr){9-11}
			  & $\Ximax$         & $2 \phiema_t$    & $\phiema_t$      & $0.5 \phiema_t$  & $2 \phiema_t$    & $\phiema_t$      & $0.5 \phiema_t$  & $2 \phiema_t$    & $\phiema_t$      & $0.5 \phiema_t$  \\ \midrule
			  & $92.83 \pm 0.12$ & $92.80 \pm 0.09$ & $92.74 \pm 0.27$ & $92.77 \pm 0.19$ & $93.04 \pm 0.08$ & $92.85 \pm 0.17$ & $92.80 \pm 0.02$ & $92.87 \pm 0.10$ & $92.90 \pm 0.12$ & $92.88 \pm 0.19$ \\
			\bottomrule
		\end{tabular}%
	}
\end{table}

\subsection{Results for Training Transformer on Multi30k} \label{sec:additional_results_on_transformer}
We additionally report the decentralized training results,
for a downsampled transformer models
(by the factor of $2$ w.r.t.\ the base model in~\citet{vaswani2017attention})
on Multi30k~\citep{elliott2016multi30k}.
Figure~\ref{fig:preliminary_transformer_appendix_version}
shows that the straightforward application of Adam in the decentralized manner
does encounter generalization problems,
which are attributed to the fact that the different local moment buffers (in addition to the weights)
become too diverse.
Tuning the learning rate schedule cannot address the issue of decentralized Adam,
as shown in the Figure~\ref{fig:preliminary_transformer_with_different_warmup_steps}.

\begin{figure*}[!h]
	\vspace{-0.5em}
	\centering
	\subfigure[\small
		The limitation of decentralized learning with Adam,
		caused by the different local moment buffers.
	]{
		\includegraphics[width=.475\textwidth,]{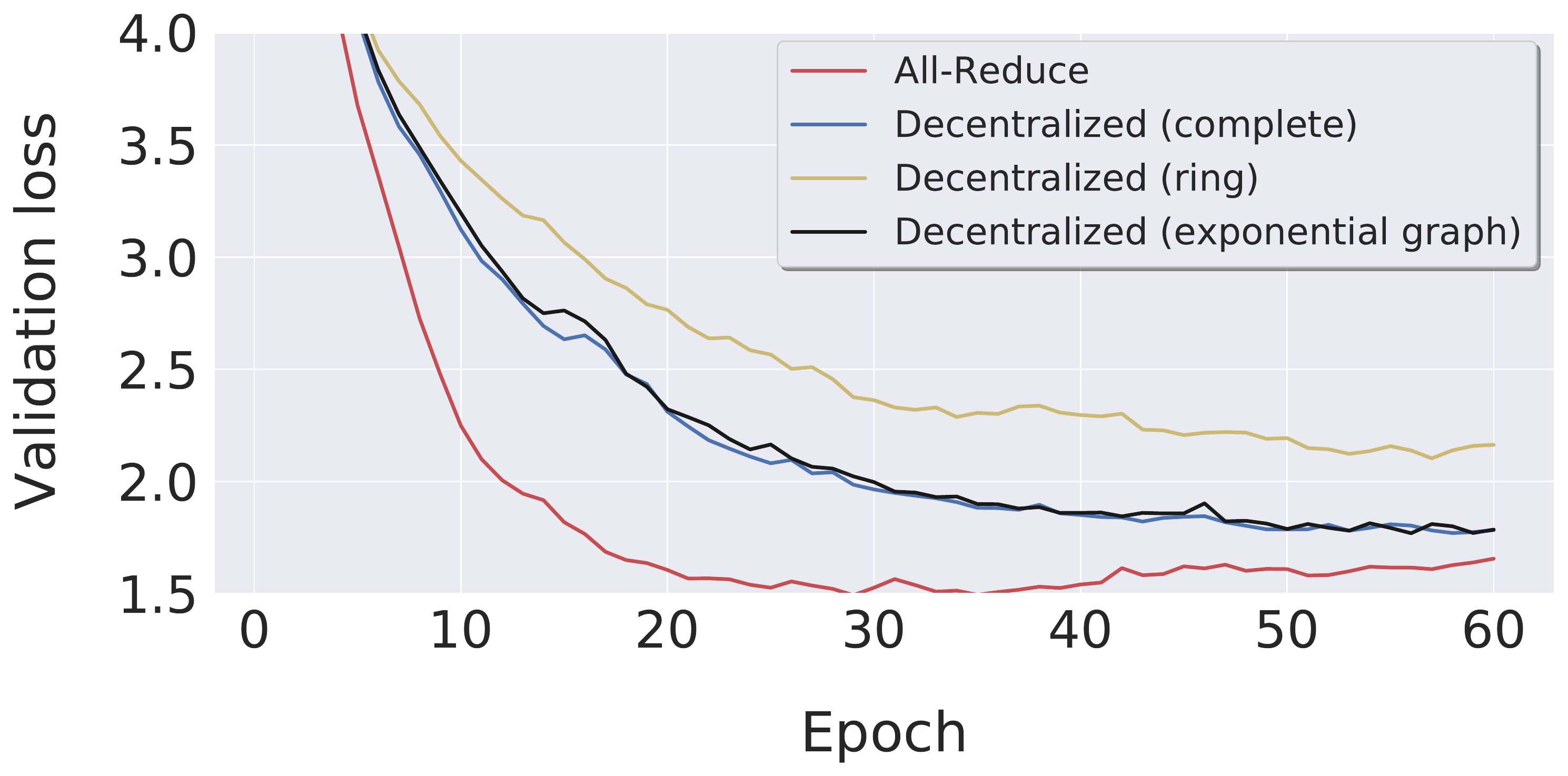}
		\label{fig:transformer_multi30k_baseline}
	}
	\hfill
	\subfigure[\small
		Tuning the learning rate cannot alleviate the issue of decentralized Adam.
	]{
		\includegraphics[width=.475\textwidth,]{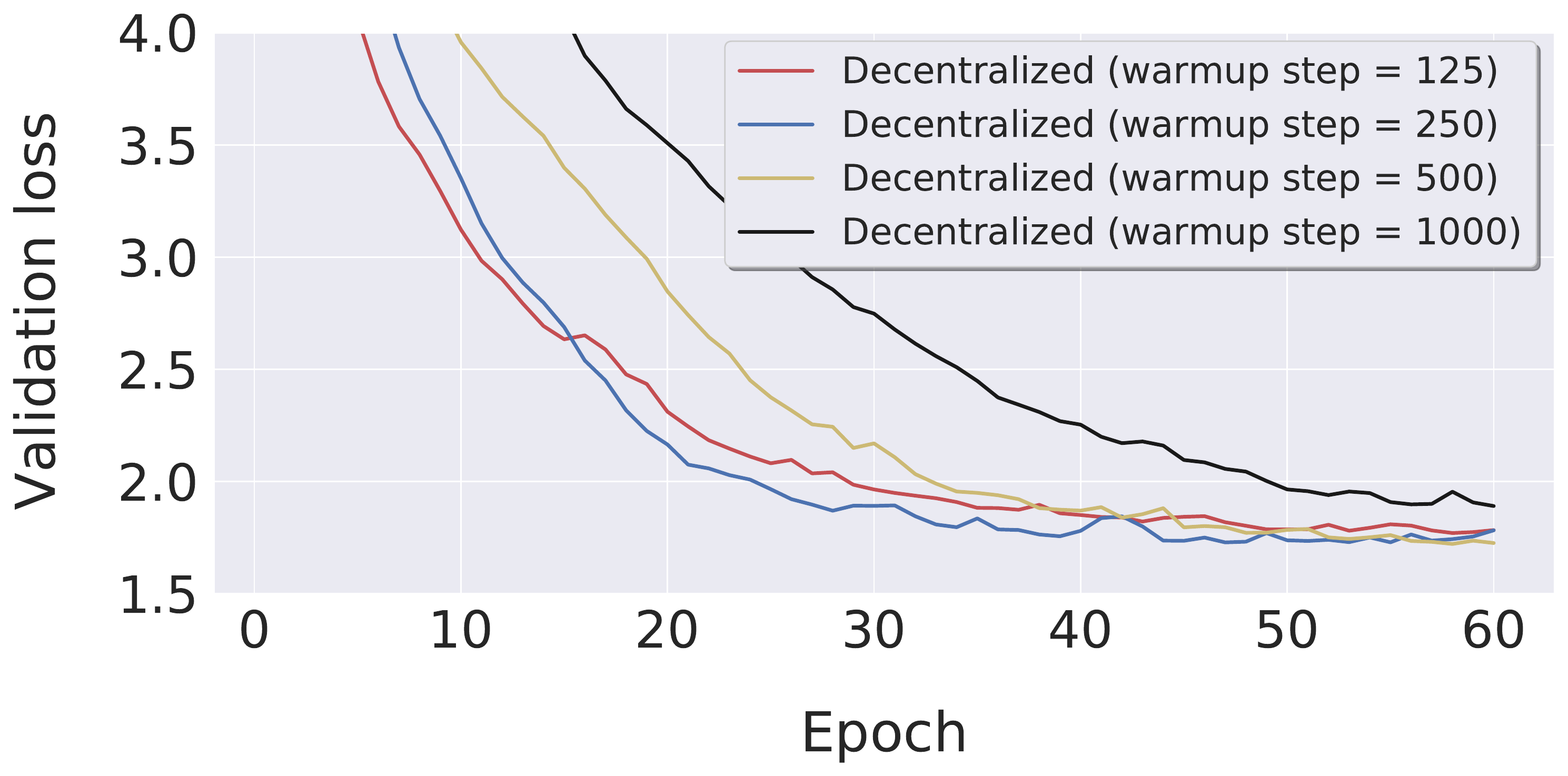}
		\label{fig:preliminary_transformer_with_different_warmup_steps}
	}
	\vspace{-1em}
	\caption{\small
		\textbf{Learning curves for training the transformer model on the Multi30k dataset} ($n\!=\!32$).
		In Figure~\ref{fig:preliminary_transformer_with_different_warmup_steps},
		we tune the the number of warmup steps as as way of tuning the learning rate,
		as the learning rate used in transformer training~\citep{vaswani2017attention}
		is deterministically controlled by the model's dimensionality, the current step index,
		and the number of warmup steps.
	}
	\vspace{-0.5em}
	\label{fig:preliminary_transformer_appendix_version}
\end{figure*}

\end{document}